%% file: main.tex
\newcommand{\jmlrprehyperref}{\RequirePackage{float}}
\documentclass{colt2026}
\usepackage{times}
\input{tex/packages}
\input{tex/macros}

\title[Feature Priming in Online Linear Regression]{Feature Priming in Online Linear Regression:
Sparse-Regret Lower Bounds and Tight Coordinatewise Rates}
\coltauthor{
\Name{Huibo Xu}\thanks{Equal contribution.}
\and \Name{Shi Fu}\footnotemark[1]
\and \Name{Qixin Zhang}
\and \Name{Dacheng Tao}\thanks{Corresponding author.}\\
\addr Nanyang Technological University, Singapore
}
\jmlrproceedings{}{Preprint}
\jmlrvolume{}
\jmlrpages{}
\jmlryear{2026}
\begin{document}
\maketitle

\begin{abstract}
\input{sections/00_abstract}
\end{abstract}

\input{sections/01_introduction}

\input{sections/02_problem_setup}
\input{sections/03_main_result}
\input{sections/04_mechanism}
\input{sections/05_rank_frontier}
\input{sections/05_ridge_robustness}
\input{sections/05_activation_experiment}
\input{sections/06_related_work}
\input{sections/07_limitations}
\input{sections/08_conclusion}

\bibliography{references}

\clearpage
\appendix
\input{appendix/A_protocol}
\input{appendix/B_common_tools}
\input{appendix/C_unit_power_intro}

\input{appendix/C1_lls_hadamard}
\input{appendix/C2_univariate_hadamard}
\input{appendix/C3_pearson_hadamard}
\input{appendix/C4_unit_power_assembly}
\input{appendix/D_powered_shared}
\input{appendix/E_rank_bound}
\input{appendix/F_univariate_frontier}
\input{appendix/G_pearson_exact}
\input{appendix/H_multivariate_euclidean}
\input{appendix/G_ridge_robustness}
\input{appendix/G_activation_experiment}

\end{document}

%% file: tex/packages.tex
\usepackage{amsfonts,mathtools}
\usepackage{microtype}
\usepackage{graphicx}
\usepackage{booktabs}
\usepackage{array}
\usepackage{xcolor}
\usepackage{xurl}
\usepackage{enumitem}
\AtBeginDocument{%
}
\setlist[itemize]{leftmargin=*,topsep=3pt,itemsep=2pt}
\setlist[enumerate]{leftmargin=*,topsep=3pt,itemsep=2pt}

\renewenvironment{proof}[1][\proofname]{%
  \par\noindent{\bfseries\upshape #1\ }\ignorespaces
}{\jmlrQED}

%% file: tex/macros.tex
\newcommand{\R}{\mathbb R}
\newcommand{\one}{\mathbf 1}
\newcommand{\Reg}{\mathcal R}
\newcommand{\diag}{\operatorname{diag}}
\newcommand{\supp}{\operatorname{supp}}
\newcommand{\clip}{\operatorname{clip}}
\newcommand{\LLS}{\mathrm{LLS}}
\newcommand{\UNI}{\mathrm{uni}}
\newcommand{\PEAR}{\mathrm{Pearson}}

\newcommand{\defeq}{\mathrel{:=}}

%% file: sections/00_abstract.tex
In high-dimensional online prediction, sparse comparators motivate regret
bounds that depend on sparsity rather than ambient dimension. Feature priming
seeks such adaptation by reweighting features using past data and refitting a
minimum-norm predictor. At COLT 2023, \citet{warmuth2023priming} posed the open
problem of whether the univariate, Pearson, or multivariate priming rules admit
competitive online regret guarantees. Under the natural past-only
Moore--Penrose protocol, we establish sparse-regret lower bounds that refute
the corresponding sparse-logarithmic guarantee. The key obstruction is cheap
nuisance interpolation, which permits exact interpolation of the history while
assigning insufficient weight to the truly predictive coordinate. An exact
target-mass identity and a two-sign argument convert this obstruction into
clipped prediction loss. Hadamard constructions yield
$\Omega(\min\{T,\sqrt d\})$ clipped regret for each of the three unit-power
rules against a zero-loss one-sparse comparator. For every fixed power
$\alpha\ge1$, one shared
paired construction further yields linear regret simultaneously for all three
powered rules and selectors among them in sufficiently high dimension.
A rank upper bound is tight for powered univariate priming, even with
Euclidean-unit inputs, and for unit-power Pearson priming with coordinatewise
bounded inputs and target-preserving totalization. A separate algebraic
construction gives $\Omega(\min\{T,d^{1/4}\})$ regret for unit-power
multivariate priming under Euclidean-unit inputs. The univariate lower bound
persists under any nonnegative second-stage ridge schedule, while a paired
ridge construction yields linear lower bounds for all three powered rules.
Exploratory diagnostics on frozen language-model activations are consistent
with the same qualitative mechanism. The exact multivariate frontier remains
open.

%% file: sections/01_introduction.tex
\section{Introduction}
\label{sec:intro}

In high-dimensional online regression, a central goal is to obtain regret
bounds that adapt to comparator sparsity rather than ambient dimension.
Coordinate-sensitive methods such as multiplicative updates can achieve this
behavior \citep{kivinen1997eg,gerchinovitz2011sparsity}, whereas minimum-norm
least squares (LLS) has no intrinsic preference for the coordinates in which
sparsity is defined.  Feature priming was proposed to introduce such a
preference without giving up the simplicity of LLS
\citep{warmuth2023priming}.  It assigns each feature a univariate, Pearson, or
multivariate LLS score, rescales the corresponding design column, and performs
a second minimum-norm fit.

\citet{warmuth2023priming} observed that these rules can recover sparse
targets faster than vanilla LLS and asked whether any of them admits a
competitive online regret bound.  Their comparison with multiplicative
updates recalls $O(k\log(n/k))$ regret when the target averages $k$ bounded
features.  We therefore ask whether one of the priming rules can guarantee
$O(k\log(ed/k))$ regret uniformly over comparators supported on at most $k$
coordinates.  Under the natural nonanticipating implementation that applies
the past-data map to each completed prefix, we give a clipping-robust negative
answer to this sparse-logarithmic form of the open problem. This does not preclude dimension-dependent guarantees, as the rank bound below demonstrates.

The negative result has a common geometric source.  Priming changes which
interpolant the second minimum-norm fit prefers.  When nuisance features
interpolate the history at low prime-weighted cost, the refit can match every
past label while placing little weight on the coordinate that actually
generates them.  If $\pi_1$ is the transformed target prime and $\kappa$ is
the minimum squared norm of a nuisance-only interpolant, we show that the
learned target weight is exactly
$w_1=\pi_1^2\kappa/(1+\pi_1^2\kappa)$.  A two-sign argument then converts the
missing target mass into clipped prediction loss.  Hadamard features keep
$\pi_1^2\kappa$ small for all three prime rules, yielding
$\Omega(\min\{T,\sqrt d\})$ regret on bounded sequences realizable with zero
loss by the one-sparse comparator $e_1$.  The regret is therefore linear when
$d=\Theta(T^2)$.  Pairing each Hadamard query with its negative gives a
stronger conclusion: for every fixed $\alpha\ge1$, the three powered rules
produce the same coefficient vector on every query round.  One fixed sequence
therefore defeats all three rules, as well as current-input-dependent
switchers and convex mixtures.

The failure is nevertheless constrained by the geometry of the observed
data.  A clipped predictor that interpolates its completed history can err
only when the current input adds a new row-space direction, and hence has
regret at most $4\operatorname{rank}(X_{1:T})$.  A rational triangular
construction matches this dependence for powered univariate priming,
establishing the exact order $\Theta_\alpha(\min\{T,d\})$ even when
$\lVert x_t\rVert_2\le1$.  A seeded affine-Hadamard construction also
attains $\Theta(\min\{T,d\})$ for unit-power Pearson priming on
coordinatewise bounded inputs, under target-preserving totalization for
the upper bound.  Its lower-bound sequence works for every finite
totalization.  For multivariate priming, a separate algebraic construction
gives $\Omega(\min\{T,d^{1/4}\})$ regret under Euclidean-unit inputs and
a unit-norm zero-loss comparator.  Thus its sparse-logarithmic failure
does not require input norms growing with dimension, although its exact
frontier remains unsettled.

Nor is the lower bound an artifact of taking the
exact Moore--Penrose limit.  Under ridge regularization of the transformed
second-stage coefficient, the same triangular sequence retains
$\Theta_\alpha(\min\{T,d\})$ regret for every nonnegative ridge schedule:
weak regularization leaves room for cheap nuisance interpolation, while
strong regularization shrinks the prediction toward zero.  A paired-Hadamard
variant gives linear regret for all three powered rules under past-only ridge,
with a common distributional lower bound for pre-input switchers and mixtures.

Exploratory stress tests on frozen Qwen2.5 and Qwen3.8 activations exhibit the
same observable pattern.  As nuisance capacity grows, nuisance-only
interpolation becomes cheaper, the learned target mass falls, and normalized
regret increases.  These diagnostics illustrate the geometry of the proof;
they are not evidence that either model implements feature priming. Our results concern the three one-stage priming rules at fixed $\alpha$ and
ridge regularization of their transformed second-stage coefficients.
Regularized prime estimation, iterative or distribution-dependent variants,
and general randomized algorithms beyond the stated switchers and mixtures
remain open.  The exact multivariate frontier, powered Pearson frontiers
beyond unit power, and Euclidean-unit Pearson guarantees under standard
fill conventions are also unresolved.

%% file: sections/02_problem_setup.tex
\section{Problem setup}
\label{sec:setup}

We consider online square-loss regression with coordinatewise bounded inputs
$x_t\in[-1,1]^d$.  Before round $t$, the learner has observed the completed
history $(X_{<t},y_{<t})$, where the rows of $X_{<t}$ are the previous
inputs.  It computes a coefficient vector from this history, observes $x_t$,
makes a prediction, and then receives the label.  This is the standard
nonanticipating online interpretation of the past-data priming map of
\citet{warmuth2023priming}. Labels are generated by a fixed sparse comparator.  For $S\subseteq[d]$ with
$|S|=k$, let $u_S=k^{-1}\sum_{j\in S}e_j$ and
$y_t=x_t^\top u_S$.  The comparator therefore has zero cumulative loss.  Our
lower bounds use only $k=1$ and $S=\{1\}$, so $u_S=e_1$ and
$y_t=x_{t,1}$. Given a base-prime vector $p_t\in\R^d$, the unit-power rule uses
$\pi_t=p_t$.  For a fixed $\alpha\ge1$, the powered rule instead uses
\[
\pi_{t,i}=\operatorname{sign}(p_{t,i})|p_{t,i}|^\alpha.
\]
Here $\operatorname{sign}(0)=0$.  Only the magnitudes of these multipliers
affect the prediction.  Indeed, write $\diag(\pi_t)=D_+S$ with
$D_+=\diag(|p_t|^\alpha)$, where $S_{ii}=\operatorname{sign}(p_{t,i})$ for
nonzero primes and $S_{ii}=1$ otherwise.  The identity
$(X D_+S)^\dagger=S(XD_+)^\dagger$ shows that the two sign matrices cancel
in the refit.  In particular, signed squaring and literal squaring are
prediction-equivalent.
Write $D_t=\diag(\pi_t)$.  Priming rescales the columns of the historical
design and performs a second minimum-norm fit:
\begin{equation}
w_t=D_t(X_{<t}D_t)^\dagger y_{<t},
\qquad
\widehat y_t=x_t^\top w_t.
\label{eq:primed-predictor}
\end{equation}
The empty-history prediction is zero.  We also study
$\widetilde y_t=\clip(\widehat y_t)=\max\{-1,\min\{1,\widehat y_t\}\}$.
Since $y_t\in[-1,1]$, clipping cannot increase square loss.

We consider the three base primes proposed by
\citet[Sec.~2, p.~2]{warmuth2023priming}:
\begin{enumerate}
  \item \textbf{Univariate LLS:}
  $p_{t,i}=X_{<t}(:,i)^\dagger y_{<t}
  =X_{<t}(:,i)^\top y_{<t}/\|X_{<t}(:,i)\|_2^2$, with value zero for a
  zero column.

  \item \textbf{Pearson:} whenever both empirical variances are positive,
  $p_{t,i}$ is the empirical Pearson correlation between
  $X_{<t}(:,i)$ and $y_{<t}$.

  \item \textbf{Multivariate LLS:} $p_t=X_{<t}^\dagger y_{<t}$.
\end{enumerate}
Together with Equation~\eqref{eq:primed-predictor}, these agree with the
original priming formulas whenever the statistics are defined.

Pearson correlation is undefined for histories of length at most one and
whenever either argument has zero empirical variance.  A \emph{finite
totalization} agrees with Pearson correlation whenever it is defined and
assigns an arbitrary finite, deterministic, past-only value otherwise.  It is
\emph{target preserving} if the filled target prime is nonzero whenever the
past label vector is not identically zero.  Our Pearson lower bounds hold for
every finite totalization; the rank upper bound requires target preservation.

Because $e_1$ has zero loss on all our constructions, regret equals learner
loss.  We write
$\Reg_T^{\rm clip}=\sum_{t=1}^T(\widetilde y_t-y_t)^2$ for the clipped
version.  The constraint throughout is $x_t\in[-1,1]^d$; it does not impose
$\|x_t\|_2\le1$ unless stated explicitly.  Unless a result states that its
witness is shared, the adversarial sequence may depend on the rule and
horizon but is fixed before online play.  The unit-power separation uses such
rule-specific witnesses.  The powered separation instead uses one witness for
all three rules at a common fixed exponent and holds pathwise for their
selectors.  The triangular ridge result is pointwise in the ridge schedule;
the corresponding policy-uniform Hadamard statement is distributional.  The
full quantifiers are collected in Appendix~\ref{app:protocol}.

%% file: sections/03_main_result.tex
\section{A nuisance-interpolation obstruction}
\label{sec:mechanism}

Although the three prime rules assign their weights differently, their
failures have the same geometric source.  The labels in the history can be
interpolated either by the target coordinate or by the nuisance coordinates.
If the latter representation is cheap after priming, the minimum-norm refit
places little weight on the target even though it interpolates every past
label.  This section isolates that obstruction without committing to a
particular prime rule.  Fix a completed history $X$ with $y=Xe_1$, and let
$\pi$ denote the actual
column multiplier used by the powered rule.  Write $\Pi=\diag(\pi)$,
$A=X\Pi$, $z=A^\dagger y$, $w=\Pi z$, and $J=\lVert z\rVert_2^2$.
Whenever $y\in\operatorname{col}(A)$, the vector $z$ is the least-norm
coefficient vector that interpolates $y$ in the transformed design.

\begin{lemma}[Target-mass identity]
\label{lem:target-mass}
If $y=Xe_1$ and $y\in\operatorname{col}(X\Pi)$, then
\[
w_1=\pi_1^2J.
\]
\end{lemma}

\begin{proof}
The minimum-norm solution lies in the row space of $X\Pi$, so
$z=(X\Pi)^\top\lambda=\Pi X^\top\lambda$ for some $\lambda$.  Since
$X\Pi z=y$, we have
$J=\lambda^\top X\Pi z=\lambda^\top y=(X^\top\lambda)_1$.
Thus $z_1=\pi_1J$ and $w_1=\pi_1z_1=\pi_1^2J$.
\end{proof}

The next lemma turns missing target weight into loss.  Its point is that an
arbitrary nuisance contribution cannot make the two possible target signs
simultaneously accurate, even after clipping.

\begin{lemma}[Two-sign clipping]
\label{lem:two-sign}
For every $b\in\R$ and $0\le w\le1$,
\[
\frac12\left[(\clip(b+w)-1)^2+(\clip(b-w)+1)^2\right]
\ge(1-w)^2.
\]
\end{lemma}

\begin{proof}
By symmetry, assume $b\ge0$.  If $b\le1-w$, neither prediction is clipped
and the left-hand side is $b^2+(1-w)^2$.  If
$1-w\le b\le1+w$, the first prediction clips to one and the average loss is
at least $\tfrac12(2-2w)^2$.  If $b\ge1+w$, both predictions clip to one
and the average loss is $2$.  Each case is at least $(1-w)^2$.
\end{proof}

We can now state the obstruction in terms of a single scalar.  Decompose the
transformed design as $A=[\pi_1y,B]$, where
$B=X(:,2{:}d)\diag(\pi_{2:d})$ contains the transformed nuisance columns.

\begin{theorem}[Nuisance-interpolation obstruction]
\label{thm:nuisance-obstruction}
Suppose that $y\in\operatorname{col}(B)$, and let
$\kappa=\lVert B^\dagger y\rVert_2^2$.  The second-stage fit assigns the
target coordinate weight
\begin{equation}
w_1=\frac{\pi_1^2\kappa}{1+\pi_1^2\kappa}.
\label{eq:obstruction-target-weight}
\end{equation}
For any next nuisance row $h\in[-1,1]^{d-1}$, consider the two legal examples
$x^\sigma=(\sigma,h)$ with labels $y^\sigma=\sigma$, where
$\sigma\in\{-1,1\}$.  Their average clipped loss is at least
$(1+\pi_1^2\kappa)^{-2}$, and hence one of the two signs incurs at least this
much loss.
\end{theorem}

\begin{proof}
For a fixed transformed target coefficient $z_1$, the nuisance coefficients
must satisfy $Bz_{2:d}=(1-\pi_1z_1)y$ and therefore have minimum squared norm
$\kappa(1-\pi_1z_1)^2$.  The second-stage fit consequently minimizes
$z_1^2+\kappa(1-\pi_1z_1)^2$, whose minimizer is
$z_1=\pi_1\kappa/(1+\pi_1^2\kappa)$.  This gives
Equation~\eqref{eq:obstruction-target-weight}.  The nuisance part of the next
prediction contributes the same scalar $b=h^\top w_{2:d}$ for both target
signs.  Since $1-w_1=(1+\pi_1^2\kappa)^{-1}$,
Lemma~\ref{lem:two-sign} gives the claimed loss bound.
\end{proof}

Thus the relevant quantity is not the prime of any one coordinate in
isolation, but the prime-weighted cost $\pi_1^2\kappa$ of explaining the
history through nuisance features.  The constructions in the next section
all keep this quantity bounded while adding new examples.

%% file: sections/04_mechanism.tex
\section{Hadamard instantiations of the obstruction}
\label{sec:results}

We now instantiate Theorem~\ref{thm:nuisance-obstruction} for each of the
three prime rules.  Fix a power of two $N$, let $H$ be an $N\times N$
Hadamard matrix, and take $d=N+1$.  A history of $r$ decisive examples has
the form $X=[a,H_r]$, where $a\in\{-1,1\}^r$ is both the target column and
the label vector, and $H_r$ contains $r$ rows of $H$.  If
$b=H_r^\top a$, then
$v=N^{-1}b$ satisfies $H_rv=a$.  Thus the nuisance coordinates interpolate
the labels regardless of the chosen signs.  What varies among the rules is
only the prime-weighted cost of this interpolant.

\begin{proposition}[Hadamard nuisance cost]
\label{prop:hadamard-cost}
Fix $\alpha\ge1$ and let $\pi$ be the powered column multiplier.  For the
history above, the powered univariate and multivariate rules satisfy
\(
\pi_1^2\lVert B^\dagger a\rVert_2^2
\le \frac{r^{2\alpha}}{N},
\)
where $B=H_r\diag(\pi_{2:d})$.  After $s$ completed opposite pairs, the
powered Pearson rule satisfies the same bound with $r=s$.  The Pearson claim
is independent of the finite totalization.
\end{proposition}

\begin{proof}
For the univariate rule, $p_1=1$ and $p_{j+1}=b_j/r$.  Whenever $b_j\ne0$,
the coefficient that represents $v_j=b_j/N$ in the transformed nuisance
design has magnitude
$r^\alpha/(N|b_j|^{\alpha-1})\le r^\alpha/N$; when $b_j=0$, both quantities
vanish.  This gives a feasible nuisance coefficient vector of squared norm at
most $r^{2\alpha}/N$, and $\pi_1=1$.

For the multivariate rule, the base prime is
$p=(r,b)/(N+r)$.  Representing the same vector $v$ in the transformed
nuisance design costs at most $(N+r)^{2\alpha}/N$, because every nonzero
$b_j$ is an integer.  Multiplication by
$\pi_1^2=[r/(N+r)]^{2\alpha}$ leaves the bound $r^{2\alpha}/N$. On a history consisting of $s$ opposite pairs, all empirical means vanish.
The Pearson base primes are then exactly the univariate primes computed from
one representative of each pair.  Opposite duplication does not change the
minimum-norm coefficient, so the univariate calculation applies with $r=s$.
All active variances are positive, and hence no filled Pearson value enters
the calculation.
\end{proof}

At unit power, keeping $r\le\sqrt N/2$ makes the scalar in
Theorem~\ref{thm:nuisance-obstruction} at most $1/4$.  At least one of the two
legal target signs therefore produces constant clipped loss.  Choosing such
a sign recursively fixes the entire sequence before online play; for Pearson
priming, the chosen query is followed by its negative so that the next charged
round again begins from a paired history.  Truncation, zero-padding, and the
rounding from Hadamard dimensions to arbitrary $(T,d)$ yield the following
separation.  Appendix~\ref{app:unit-power} gives the compiler and exact
constants.

\begin{theorem}[Sparse-regret separation at unit power]
\label{thm:main}
For every $T\ge1$ and $d\ge5$, each of the unit-power multivariate LLS-prime
and univariate-prime predictors admits a deterministic realizable sequence
such that
\[
\Reg_T^{\rm clip}
\ge \frac{1}{15}\min\{T,\sqrt d\}.
\]
For the Pearson-prime predictor and $T\ge4$, a single deterministic sequence
gives the same lower bound simultaneously for every finite totalization.  In
all cases the comparator is $e_1$ and has zero cumulative loss.
\end{theorem}

\begin{corollary}[Failure of uniform sparse-logarithmic regret]
\label{cor:no-sparse-log}
Under the past-only protocol of Section~\ref{sec:setup}, none of the three
unit-power rules admits a universal constant $C$ such that
$\Reg_T^{\rm clip}\le Ck\log(ed/k)$ for every $d,k,T$, support $S$, and
bounded sequence realizable by $u_S$.  This already fails for $k=1$; for
Pearson priming the conclusion is uniform over every finite totalization.
The same is therefore true for the unmodified raw predictions.
\end{corollary}

Indeed, take $d=\Theta(T^2)$ in Theorem~\ref{thm:main}.  Its lower bound is
linear in $T$, whereas the proposed benchmark is only $O(\log T)$.  Clipping
cannot increase square loss, so the raw predictor cannot have a smaller
worst-case loss.  The witnesses in Theorem~\ref{thm:main} may depend on the
rule; the next result gives one sequence that works for all three and for
every fixed prime power.

After $s$ completed query--recovery pairs, let $b=H_s^\top a$.  The three
base-prime vectors obey
\[
p^{\UNI}=p^{\PEAR}=(1,b/s),
\qquad
p^{\LLS}=\frac{s}{N+s}\,p^{\UNI}.
\]
Raising them to a common power preserves the positive global factor, which
cancels from the second minimum-norm fit.  The three powered rules therefore
return exactly the same coefficient vector on every charged query.  A
\emph{three-rule selector} may switch among these rules, possibly at random
and after seeing the current input, or take a convex mixture of either their
coefficient vectors or their individually clipped predictions.

\begin{theorem}[One shared witness for powered rules]
\label{thm:powered}
Fix $\alpha\ge1$.  For every integer $M\ge1$, there is a power of two $N$
satisfying $8M^{2\alpha}\le N<16M^{2\alpha}$.  Set $d=N+1$ and $T=2M$.
There exists a deterministic realizable sign sequence, fixed before play and
independent of the prime rule, selector, and finite Pearson totalization, such
that each of the three clipped powered rules and every three-rule selector
satisfies
\[
\Reg_T^{\rm clip}\ge\frac{49M}{64}=\frac{49T}{128}
\]
for every realization of its internal randomness.
\end{theorem}

The witness alternates a charged query with its exact negative.  The paired
state makes the three refits identical, while Proposition~\ref{prop:hadamard-cost}
keeps the target weight bounded away from one; the two-sign argument then
selects a common adverse sign.  The construction is pathwise, so allowing a
selector to depend on the current input or on its random tape does not help.
There is also an oblivious distributional version: independent Rademacher
query signs give expected clipped regret at least $49M/64$, even after
conditioning on the selector's random tape.  Appendix~\ref{app:powered}
contains the exact equivalence and compilation arguments.

Since $d=\Theta_\alpha(T^{2\alpha})$, Theorem~\ref{thm:powered} rules out the
sparse-logarithmic benchmark for every fixed $\alpha\ge1$, including the
squared-prime choice $\alpha=2$.  The exponent is common and fixed throughout
the sequence; the theorem does not cover procedures that tune it online or
add an external correction to the primed prediction.

%% file: sections/05_rank_frontier.tex
\section{Rank controls loss and dimension dependence}
\label{sec:frontier}

The Hadamard constructions can keep nuisance interpolation cheap only while
the design continues to acquire new directions.  More generally, a predictor
that interpolates the completed history is already forced to be correct on
its row span.  Rank therefore measures the geometric resource consumed by
each loss-producing round.

\begin{theorem}[Rank bound for history interpolation]
\label{thm:rank-adaptive}
Let $x_t\in[-1,1]^d$, and suppose that a fixed comparator $u\in\R^d$
realizes the labels, $y_t=x_t^\top u\in[-1,1]$.  Assume that the online
predictor outputs $w_t$ satisfying $X_{<t}w_t=y_{<t}$ and predicts
$\clip(x_t^\top w_t)$.  Then
\[
\sum_{t=1}^T\bigl(\clip(x_t^\top w_t)-y_t\bigr)^2
\le4\operatorname{rank}(X_{1:T})
\le4\min\{T,d\}.
\]
For every fixed $\alpha\ge1$, the bound applies to the powered univariate and
multivariate priming rules on sequences realizable by $e_1$.  It also applies
to powered Pearson priming under any target-preserving totalization.
\end{theorem}

\begin{proof}
Suppose that $x_t^\top$ belongs to the row span of $X_{<t}$, and write
$x_t^\top=c^\top X_{<t}$.  Interpolation and realizability give
$x_t^\top w_t=c^\top X_{<t}w_t=c^\top y_{<t}
=c^\top X_{<t}u=x_t^\top u=y_t$.
Positive loss is therefore possible only when $x_t$ adds a new row-space
direction.  There are at most $\operatorname{rank}(X_{1:T})$ such rounds,
and clipping bounds the squared loss on each round by four.

It remains to verify interpolation for the listed priming rules.  Write
$X=X_{<t}$ and $y=y_{<t}$.  If $y=0$, the second-stage least-norm fit is
zero.  Otherwise, the target prime of the powered univariate rule is one.
For the multivariate rule, let $p=X^\dagger y$ and set
$\pi_i=\operatorname{sign}(p_i)|p_i|^\alpha$.  Defining $z_i=p_i/\pi_i$
on $\supp(p)$ and $z_i=0$ elsewhere gives
$X\diag(\pi)z=Xp=y$.  For a target-preserving Pearson totalization,
$\pi_1\ne0$, and the transformed target column $\pi_1y$ alone interpolates
$y$.  Thus each second-stage predictor interpolates its history.
\end{proof}

The bound adapts to the realized geometry: if all inputs lie in an
$r$-dimensional subspace, regret is at most $4r$, independently of the
horizon.  It is not a sparsity bound.  Even when $e_1$ generates every label,
the observed design may have rank $\min\{T,d\}$. For powered univariate priming, this rank dependence is unavoidable.  A
triangular design creates one new direction on each loss-producing round and
matches the upper bound up to a constant depending only on $\alpha$.

\begin{theorem}[Powered-univariate frontier]
\label{thm:univariate-exact}
Fix $\alpha\ge1$, and let $c_\alpha$ be the smaller of $1/8$ and
$(9/16)(5/8)^{4\alpha-4}$.  For every $d,T\ge1$, with
$m=\min\{d,T\}$, there is a fixed rational sequence in $[-1,1]^d$, realizable
with zero loss by $e_1$, on which the clipped power-$\alpha$ univariate rule
has regret at least $c_\alpha m$.  At unit power, the same sequence satisfies
the sharper lower bound $1+\lfloor m/2\rfloor$.  There is also a sequence
satisfying $\lVert x_t\rVert_2\le1$ with regret at least $c_\alpha m/4$.
\end{theorem}

The active $m\times m$ block has $y_t=x_{t,1}=1$ and, for $2\le k\le m$,
\[
x_{t,k}=
\begin{cases}
1/(k-1), & t<k,\\
-1,      & t=k,\\
0,       & t>k.
\end{cases}
\]
Before round $t$, expired coordinates have zero prime, while future
coordinates retain reciprocal scales.  With $\beta=2\alpha-2$, the
least-norm split yields
\[
1-\widehat y_t
=\frac{t(t-1)^\beta}{1+\sum_{j=t-1}^{m-1}j^\beta}.
\]
This error is bounded away from zero on a constant fraction of the final
rounds, giving the claimed $\Omega_\alpha(m)$ loss.  At $\alpha=1$ the same
expression yields the sharper constant, and scaling the active rows by one
half enforces $\lVert x_t\rVert_2\le1$ at only constant-factor cost.
Appendix~\ref{app:univariate} contains the pseudoinverse calculation and the
endpoint estimates.

Together with Theorem~\ref{thm:rank-adaptive}, the lower bound shows that the
clipped worst-case regret of powered univariate priming is
$\Theta_\alpha(\min\{T,d\})$, both with coordinatewise bounded inputs and
under the additional Euclidean normalization.  The same coordinatewise
order holds for unit-power Pearson priming, but its proof requires a
different construction: constant labels would make every Pearson
correlation undefined and leave the result dependent on the fill convention.

\begin{theorem}[Exact Pearson frontier at unit power]
\label{thm:pearson-exact-frontier}
For every $T,d\ge1$, there is one fixed rational sequence satisfying
$x_t\in[-1,1]^d$ and $y_t=x_{t,1}$ such that, for every finite deterministic
past-only Pearson totalization, the unit-power clipped predictor has
\[
 \Reg_T^{\rm clip}(e_1)\ge\frac1{131584}\min\{T,d\}.
\]
For each target-preserving totalization, its worst-case clipped regret over
this realizable class is therefore $\Theta(\min\{T,d\})$.
\end{theorem}

The construction seeds the history with balanced label pairs and then adds
fresh Hadamard directions.  Each nuisance feature combines a label-aligned
component with a blockwise mean component.  On every charged Hadamard
query, all active correlations are defined.  An orthogonal change of coefficient
coordinates removes the block-mean component from the fit, leaving a
weighted Hadamard interpolation problem whose next prediction has magnitude
less than $1/2$.  A constant fraction of the dimension can thus be charged
without consulting the totalization.  Appendix~\ref{app:pearson-exact-frontier}
gives the compression, spectral estimate, and small-size cases.
The matching upper bound uses Theorem~\ref{thm:rank-adaptive}; without
target preservation, zero-filling an undefined target prime can instead
give loss $T$ on $x_t=y_t=1$ in dimension one.

Euclidean normalization imposes a different constraint.  The Pearson
construction above is coordinatewise bounded, not Euclidean-unit.  For
multivariate priming, the following result establishes a dimension-growing
lower bound even under unit Euclidean inputs.

\begin{theorem}[Euclidean-unit multivariate lower bound]
\label{thm:multivariate-euclidean}
There is a universal $c_{\rm M}>0$ such that, for every $T,d\ge1$, a fixed
algebraic sequence with $\|x_t\|_2\le1$ and $y_t=x_{t,1}$ makes the
unit-power multivariate priming rule incur
\[
 \Reg_T^{\rm clip}(e_1)\ge c_{\rm M}\min\{T,d^{1/4}\}.
\]
The comparator $e_1$ has unit norm and zero cumulative loss.
\end{theorem}

Two orthonormal examples realize a state that spreads predictive weight
over paired nuisance coordinates.  A fixed orthonormal zero-label tail
keeps the first-stage prime unchanged, while alternating pair scales force
constant clipped loss on a positive fraction of the tail.  This yields
linear regret with $d=\Theta(T^4)$, so large input norms are not necessary
for multivariate sparse-logarithmic failure either.  The complete algebraic
construction and its Schur-complement calculation are in
Appendix~\ref{app:multivariate-euclidean}.  Unlike the univariate and
coordinatewise unit-power Pearson results, this is a lower bound rather
than an exact multivariate frontier.

%% file: sections/05_ridge_robustness.tex
\section{Ridge changes the failure mode, not the rate}
\label{sec:ridge}

A natural response to the lower bounds is to regularize the second refit.  We
consider the standard ridge estimator in the transformed coordinates.  With
$A_t=X_{<t}D_t$ and $\lambda_t>0$, it returns
\[
z_t^{\lambda_t}
=A_t^\top(A_tA_t^\top+\lambda_t I)^{-1}y_{<t},
\qquad
w_t^{\lambda_t}=D_tz_t^{\lambda_t};
\]
at $\lambda_t=0$ we recover $z_t^0=A_t^\dagger y_{<t}$.  Thus the penalty is
on the transformed coefficient $z$, not on the transformed-back coefficient
$w$. The nuisance-interpolation mechanism has an exact ridge analogue.  For one
history, split $A=[\pi_1y,B]$ and define
\begin{equation}
q_\lambda=y^\top(BB^\top+\lambda I)^{-1}y,
\qquad
w_1^\lambda=\frac{\pi_1^2q_\lambda}{1+\pi_1^2q_\lambda}.
\label{eq:ridge-target-mass}
\end{equation}
If $Bu=y$, then $q_\lambda\le\lVert u\rVert_2^2$, and $q_\lambda$ decreases
with $\lambda$.  Ridge therefore cannot restore target mass that is already
suppressed by a cheap nuisance interpolant.

\begin{theorem}[Ridge-robust univariate frontier]
\label{thm:ridge-univariate}
Fix $\alpha\ge1$ and let $c_\alpha$ be as in
Theorem~\ref{thm:univariate-exact}.  For every $d,T\ge1$, the same fixed
rational triangular sequence satisfies, simultaneously for every
nonnegative ridge schedule $(\lambda_t)$ with $\lambda_t<\infty$,
\[
\Reg_T^{\rm clip}\ge c_\alpha\min\{T,d\}.
\]
At unit power the lower bound is
$1+\lfloor\min\{T,d\}/2\rfloor$.  Under $\lVert x_t\rVert_2\le1$ and
$\lVert e_1\rVert_2=1$, the bound remains
$(c_\alpha/4)\min\{T,d\}$.  These statements hold pointwise in the schedule,
even if $\lambda_t$ is randomized or selected after observing $x_t$.
\end{theorem}

The phase scale is explicit.  Before triangular round $t\ge2$, let
$s=t-1$ and
$S_t=1+\sum_{j=t-1}^{m-1}j^{2\alpha-2}$, where
$m=\min\{T,d\}$.  The entire ridge prediction is
\begin{equation}
\widehat y_t^\lambda
=\rho_t(\lambda)\widehat y_t^0,
\qquad
\rho_t(\lambda)
=\frac{sS_t}{\lambda+sS_t}
=\frac{1}{1+\lambda/\lambda_c(t)},
\qquad
\lambda_c(t)=sS_t.
\label{eq:ridge-phase-scale}
\end{equation}
For $\lambda\ll\lambda_c(t)$, the nuisance-driven Moore--Penrose prediction
persists; for $\lambda\gg\lambda_c(t)$, the prediction approaches zero and
underfits the label one.  The crossover is therefore between two failure
mechanisms, not between failure and success. The proof uses the rank-one transformed triangular history to derive
Equation~\eqref{eq:ridge-phase-scale}, then observes that shrinking a raw
prediction $q\le1$ toward zero leaves clipped loss at least
$\min\{(1-q)^2,1\}$.  Appendix~\ref{app:ridge} gives the complete proof. The resolvent bound following Equation~\eqref{eq:ridge-target-mass} also
preserves the paired-Hadamard obstruction for all three rules, although the
regularized coefficient vectors no longer coincide.

\begin{theorem}[Ridge-robust separation for all three rules]
\label{thm:ridge-three-rule}
Fix $\alpha\ge1$ and $M\ge1$, choose a power of two $N$ with
$8M^{2\alpha}\le N<16M^{2\alpha}$, and set $T=2M$ and $d=N+1$.
For each powered prime rule and each deterministic past-only ridge policy,
there is a deterministic realizable paired-Hadamard sequence on which
$\Reg_T^{\rm clip}\ge32T/81$.  The sequence may depend on the rule and
policy.  Under independent Rademacher query signs, the same bound holds in
expectation for every possibly randomized past-only ridge policy and for
every pre-input switcher or convex mixture of the three ridge-regularized
rules.  The Pearson statements are uniform over every finite totalization.
\end{theorem}

The deterministic compiler requires the ridge value to be fixed before the
current target sign is revealed.  By contrast,
Theorem~\ref{thm:ridge-univariate} is pointwise in the schedule and therefore
allows the univariate ridge value to depend on the current input.  The common
three-rule statement in Theorem~\ref{thm:ridge-three-rule} is distributional.

%% file: sections/05_activation_experiment.tex
\section{The obstruction in frozen language-model representations}
\label{sec:activation-experiment}

As an exploratory diagnostic, we test the geometric mechanism on post-gating
MLP activations from frozen
Qwen2.5-7B-Instruct layers 6, 13, and 20, using Alpaca instruction records
\citep{yang2024qwen25,taori2023alpaca}.  Calibration data select six
nondegenerate target neurons per layer.  For target $j$, the bounded label is
$y_t=x_{t,j}$, so $e_j$ has zero loss and nested sets of other neurons serve as
nuisance coordinates.  We repeat the calibration-only selection, criterion,
horizon, registered dimensions, and ordering protocol on frozen Qwen3.8-27B
layers 14, 30, and 46 at corresponding relative depths
\citep{qwen38modelcard}.  The diagnostic is deliberately synthetic: labels
are individual activation coordinates, and the input order is compiled
offline to stress the priming maps.

Figure~\ref{fig:qwen38-mechanism} gives the registered Qwen3.8 dimension scan.
Increasing nuisance dimension raises normalized regret, lowers target mass, and makes
cheap nuisance-only certificates common.  On six fixed targets, target-only
median normalized regret is $0.029$ for all three rules; with $2{,}047$ real
nuisance neurons ($2{,}048$ coordinates including the target) it is $1.238$,
$0.303$, and $0.587$ for univariate, multivariate, and
Pearson priming.  Independently permuting each nuisance coordinate leaves
medians $0.799$, $0.306$, and $0.555$, so the effect does not require the
model's joint neuron correlations.  The Qwen2.5 dimension sweep exhibits the
same mechanism.
\begin{figure}[t]
  \centering
  \includegraphics[width=\linewidth]{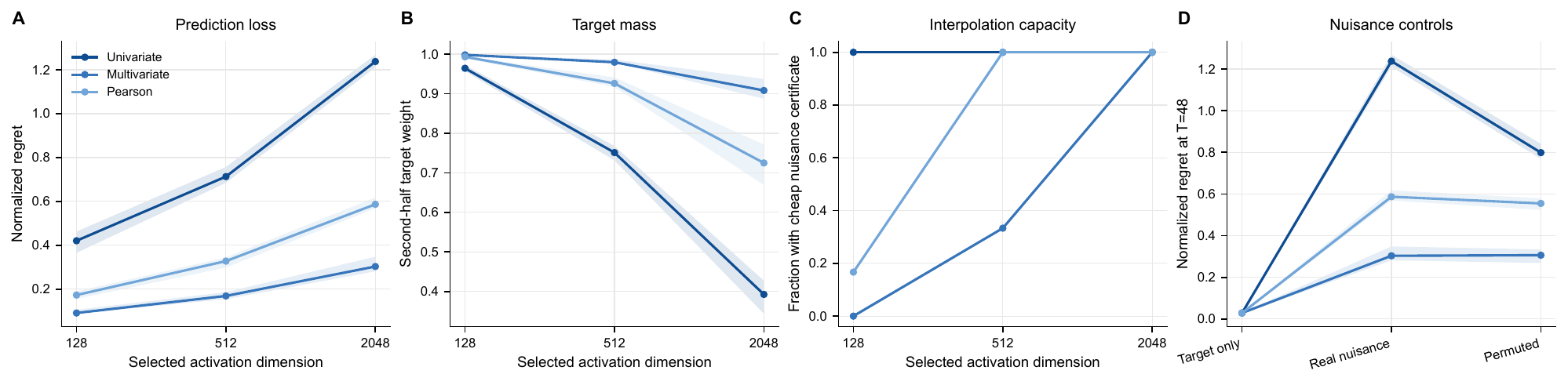}
  \caption{\textbf{Qwen3.8-27B reproduces the nuisance-capacity mechanism.}
  Across six registered targets, increasing activation dimension raises
  normalized regret (A), lowers target weight (B), and makes cheap nuisance certificates
  common (C).  Target-only, real-nuisance, and independently permuted-nuisance
  controls appear in D.  Lines and bands in A, B, and D are target-level
  medians and interquartile ranges; C reports the target fraction satisfying
  the certificate.}
  \label{fig:qwen38-mechanism}
\end{figure}
Appendix~\ref{app:activation-experiment} gives the cross-generation table,
complete Qwen2.5 diagnostic, protocol details, and numerical checks.  These are
exploratory diagnostics of the proof's geometry; they do not test whether
either model internally implements feature priming.

%% file: sections/06_related_work.tex
\section{Related work}
\label{sec:related}

\citet[Sec.~2, p.~2]{warmuth2023priming} define the prime--refit map with
univariate, Pearson, and multivariate LLS primes.  Their Open Problem~1 asks
whether any of these priming methods admits a competitive regret bound, without
fixing dimension or horizon dependence.  Because the note first compares with
$O(k\log(n/k))$ regret for averages of $k$ bounded features, we study the
corresponding uniform sparse-logarithmic interpretation and answer it
negatively.  Dimension-dependent guarantees remain possible.
Theorem~\ref{thm:powered} covers their squared-prime variant, and
Section~\ref{sec:setup} proves that signed and literal squaring give the same
second-stage prediction.

\paragraph{Sparse regret and adaptive geometry.}
The sparse-logarithmic benchmark reflects the classical contrast between
coordinatewise and Euclidean updates.  Exponentiated-gradient methods can
exploit comparators supported on few coordinates \citep{kivinen1997eg}, and
exponential weighting with truncation gives deterministic sparsity regret
bounds for online regression \citep{gerchinovitz2011sparsity}.  By contrast,
competitive ridge and least-squares methods naturally yield norm- and
dimension-dependent guarantees
\citep{cesabianchi1996quadratic,vovk1997competitive,forster1999relative,
azoury2001relative,vovk2001statistics}.  Second-order aggregation and adaptive
regularization provide other data-dependent guarantees
\citep{mcmahan2010adaptive,duchi2011adaptive,gaillard2014secondorder}.
Restricted feature observation introduces additional constraints in budgeted
statistical learning \citep{cesabianchi2011partial,hazan2012limited} and online
sparse regression \citep{foster2016online,kale2017adaptive,ito2018partial,
li2025polynomial}.  These alternatives change
the update rule, regularizer, feedback model, or structural assumptions; they
do not analyze the same closed-form prime--refit map.

In batch sparse recovery, explicit $\ell_1$ regularization, constrained
selectors, greedy recovery, and iteratively reweighted least squares provide
well-developed alternatives
\citep{tibshirani1996lasso,candes2007dantzig,needell2009cosamp,
chartrand2008reweighted,daubechies2010irls}.  Sparse solutions can also arise
from the implicit bias of a reparameterized least-squares problem under
additional design and optimization conditions \citep{vaskevicius2019implicit}.

\paragraph{Rotation-invariance and symmetry lower bounds.}
The distinction between coordinate-sensitive and rotation-invariant methods is
classical in sparse learning \citep{ng2004rotation,warmuth2005span}; more
generally, optimization geometry determines which interpolating solution an
underdetermined problem selects \citep{gunasekar2018geometry}.
The closest geometric predecessor is the Hadamard separation of
\citet{warmuth2021spindly}, which gives a sample-efficiency lower bound for
gradient-trained networks with fully connected input layers and
rotation-invariant initialization.
Feature priming is not rotation invariant: it deliberately rescales the
observed coordinates.  We retain the sparse-target Hadamard geometry but
analyze three concrete, non-invariant prime--refit maps under adversarial online
ordering and clipped square loss.  The target-mass identity and clipping lemma
provide the bridge to all three rules, while the matching rank upper bound and
rational triangular construction yield the exact powered-univariate frontier.
Related symmetry lower bounds have
also been proved for sparse logistic models with hard labels under i.i.d.
Gaussian covariates \citep{ghosh2026hardlabels}; those results concern
statistical excess risk for an invariant algorithm class rather than online
square loss for these three non-invariant algorithms.

\paragraph{Reweighting beyond one-stage priming.}
Negative results for the three one-stage rules do not preclude successful
reweighting with a different outer estimator.  Mirror-descent updates can
themselves be represented through suitable gradient-descent
reparameterizations \citep{amid2020mirror}.  In their overconstrained noisy
sparse-target model, \citet{warmuth2025rotation} establish statistical upper bounds
for non-rotation-invariant procedures including an LLS-primed estimator
followed by ridge regression at a specified regularization level.  The linear
recursive feature machine of \citet{radhakrishnan2025linrfm} instead alternates
reweighting and refitting, reduces to an iteratively reweighted least-squares
variant in the linear setting,
and has recovery guarantees for structured statistical problems; it builds on
the average-gradient-outer-product feature-learning mechanism of
\citet{radhakrishnan2024mechanism}.  These methods use iteration or
distributional assumptions absent from our adversarial online protocol.
Theorem~\ref{thm:ridge-univariate} rules out this repair on an adversarial
triangular sequence, uniformly over the regularization level; it does not
conflict with statistical guarantees under prescribed design and noise
assumptions.

%% file: sections/07_limitations.tex
\section{Scope and limitations}
\label{sec:limitations}

We study the three one-stage rules at fixed $\alpha\ge1$, together with
standard ridge regularization of their transformed second-stage coefficient.
A penalty on the transformed-back coefficient $w$, a regularized first-stage
prime, iterative or alternative-prime procedures, and selectors that add
other predictors or corrections remain outside scope.
The paired-Hadamard witness is shared across rules, not powers or horizons.
Its selector guarantee concerns only the three rules at one common fixed power;
it permits mixtures of either their coefficients or their clipped predictions.
For ridge, the deterministic Hadamard compiler requires a pre-sign
regularization level; its policy-uniform version is distributional.  The
triangular theorem is pointwise and permits current-input-dependent
regularization.

The shared witness uses $d=\Theta_\alpha(T^{2\alpha})$, which is not claimed sharp.
The triangular family closes the powered-univariate frontier at
$\Theta_\alpha(\min\{T,d\})$.  The seeded construction also closes the
coordinatewise unit-power Pearson frontier at $\Theta(\min\{T,d\})$,
but does not establish the corresponding powered or Euclidean-unit
frontiers.  The coordinatewise unit-power multivariate bounds retain a
$\sqrt d$--$d$ gap.  The rank-adaptive guarantee is
ensured for Pearson by target preservation; more generally, its exact
condition is interpolation by the transformed historical design.

The half-scaled triangular family satisfies $\|x_t\|_2,\|e_1\|_2\le1$, so its
exact frontier survives both Euclidean constraints.  The Hadamard rows have
norm $\sqrt d$, but the separate algebraic construction proves a
unit-power multivariate lower bound of
$\Omega(\min\{T,d^{1/4}\})$ with Euclidean-unit inputs.  We do not claim
a matching multivariate upper bound or extend this construction to ridge
or powers above one.  A Euclidean-unit Pearson separation that is uniform
over finite totalizations remains open; allowing a specially chosen fill
is a different question from fixing the standard zero-fill convention.

The frozen-activation experiments test the three downstream priming maps on
synthetic one-neuron labels inside two real representations.  They do not test
Qwen's generation mechanism, natural semantic labels, or whether a Transformer
implements these online updates internally.

%% file: sections/08_conclusion.tex
\section{Conclusion}

Under the natural past-only Moore--Penrose protocol, none of the three
one-stage feature-priming rules admits uniform sparse-logarithmic regret.  The
common obstruction is cheap nuisance interpolation: if $\kappa$ is the
least transformed cost of interpolating the history without the target, then
the learned target weight is exactly
$\pi_1^2\kappa/(1+\pi_1^2\kappa)$.  Hadamard features keep this quantity
small for all three prime definitions, producing
$\Omega(\min\{T,\sqrt d\})$ clipped regret against a zero-loss one-sparse
comparator.  On paired histories the three powered refits coincide, so one
sequence also defeats current-input-dependent switchers and convex mixtures
for every fixed $\alpha\ge1$.

The same analysis also identifies the limit of the failure.  Every clipped
history-interpolating predictor has regret at most
$4\operatorname{rank}(X_{1:T})$, because an error requires a new row-space
direction.  A Euclidean-normalized triangular sequence matches this
dependence for powered univariate priming, giving
$\Theta_\alpha(\min\{T,d\})$ worst-case regret.  A seeded construction
also attains the rank order for unit-power Pearson priming on coordinatewise
bounded inputs, with target preservation for the matching upper bound.
For unit-power multivariate priming, an algebraic Euclidean-unit sequence
gives $\Omega(\min\{T,d^{1/4}\})$ regret, separating its failure from
large input norms without claiming a tight frontier.  The univariate frontier survives
transformed-coordinate ridge: small regularization preserves nuisance
prediction, while large regularization underfits toward zero.  More generally,
paired Hadamard histories give linear regret for all three powered rules under
past-only ridge policies, with a common distributional obstruction for
pre-input switchers and mixtures. Thus successful empirical
feature recovery does not by itself imply sparse online adaptation: nuisance
features determine which interpolant priming selects, while rank determines
how long that choice can remain costly.  The exact multivariate frontier,
and the powered and Euclidean extensions of the Pearson frontier, remain
open.

%% file: appendix/A_protocol.tex
\section{Formal protocol and conventions}
\label{app:protocol}

The appendices follow the proof dependencies.  Appendix~\ref{app:core}
collects the common algebraic tools; Appendices~\ref{app:unit-power} and
\ref{app:powered} prove Theorems~\ref{thm:main} and \ref{thm:powered};
Appendix~\ref{app:rank} proves the rank upper bound; and
Appendix~\ref{app:univariate} proves the matching powered-univariate
frontier.  Appendices~\ref{app:pearson-exact-frontier} and
\ref{app:multivariate-euclidean} give the seeded Pearson frontier and the
Euclidean-unit multivariate lower bound.  Ridge proofs and the
frozen-activation diagnostics follow in Appendices~\ref{app:ridge} and
\ref{app:activation-experiment}.

\subsection{Quantifiers}

Fix a deterministic priming rule.  A regret upper bound of sparse logarithmic
form would assert that there is a universal constant $C$ such that for every
$d,k,T$, every support $S\subseteq[d]$ with $|S|=k$, and every possibly
adaptive sequence $x_t\in[-1,1]^d$, the past-only predictor satisfies
\[
\Reg_T\le Ck\log(ed/k)
\]
when $y_t=x_t^\top u_S$ and
$u_S=k^{-1}\sum_{j\in S}e_j$.  To refute this statement it is enough, for
each of an unbounded set of horizons, to provide a possibly horizon-dependent
dimension, support, and finite sequence whose regret grows faster than the
right-hand side.
The factor $e$ keeps the logarithm positive at $k=d$ and is immaterial to the
asymptotic benchmark recalled in the source note.

Our constructions use $S=\{1\}$.  In the rule-specific Hadamard construction,
each decisive round has two legal target signs.  Given the deterministic rule and a fixed tie-breaking
convention, we simulate its prediction on the current prefix and record the
sign with larger loss.  For any fixed horizon this finite recursion specifies
the entire sequence before the online game begins.  Thus the witness is
algorithm-dependent but oblivious after compilation: it is fixed before the
actual interaction, although it may depend on the deterministic rule being
tested.

The shared witness in Theorem~\ref{thm:powered} uses a stronger offline
compiler.  At each node of a finite paired-Hadamard sign tree, it evaluates
the two possible next query signs.  The three powered coefficient vectors are
identical at every such node, so it selects a sign under which their common
prediction incurs constant clipped loss.  Appendix~\ref{app:powered} proves the
resulting linear loss for all three and for their switchers and mixtures.
The selected sequence is therefore independent of which of the three rules is
subsequently run and remains fixed throughout online play.

\subsection{Timing}

At round $t$:
\begin{enumerate}
  \item the learner has access only to $(X_{<t},y_{<t})$;
  \item the prime $p_t$ and $w_t$ are computed from that history;
  \item $x_t$ is presented and the learner predicts $x_t^\top w_t$ (or its
  clipped value);
  \item the label $y_t=x_{t,1}$ is revealed and appended to the history.
\end{enumerate}
No displayed weight uses $X_{\le t}$ when predicting at time $t$.

\subsection{Three-rule selectors}

The selector extension in Theorem~\ref{thm:powered} concerns only the three
powered priming rules at one common fixed exponent; it does not permit an
additional predictor or correction term.  We formalize both mixture semantics
covered by the theorem.

\begin{definition}[Three-rule selector]
\label{def:three-rule-selector}
Fix $\alpha\ge1$ and let
$\mathcal R=\{\UNI,\LLS,\PEAR\}$.  At round $t$, rule $R$ has coefficient
$w_t^R$, computed from the completed history as in
Equation~\eqref{eq:primed-predictor}.  A possibly randomized selector has a
random tape $\xi$ and, after observing the completed history and current input
$x_t$, chooses
\[
\lambda_t=\lambda_t(X_{<t},y_{<t},x_t,\xi)\in
\Delta_{\mathcal R}
\defeq\left\{\lambda\in[0,1]^3:\sum_{R\in\mathcal R}\lambda_R=1\right\}.
\]
The map from $(X_{<t},y_{<t},x_t,\xi)$ to $\lambda_t$ is measurable.  The
selector uses one of the following two predictions:
\begin{align}
\widetilde y_t^{\rm coef}
&=\clip\!\left(x_t^\top\sum_{R\in\mathcal R}
                  \lambda_{t,R}w_t^R\right),
\label{eq:coefficient-selector}\\
\widetilde y_t^{\rm pred}
&=\sum_{R\in\mathcal R}\lambda_{t,R}
                  \clip(x_t^\top w_t^R).
\label{eq:prediction-selector}
\end{align}
A switcher is the special case in which $\lambda_t$ is a simplex vertex.
\end{definition}

The second output already lies in $[-1,1]$, so applying a final clipping does
not change it.  In general Equations~\eqref{eq:coefficient-selector} and
\eqref{eq:prediction-selector} are different because clipping and mixing do
not commute.  Theorem~\ref{thm:powered} covers both only because the three raw
predictions coincide on every charged query round.  The definition excludes
arbitrary current-input corrections; such a broader learner could simply use
$x_{t,1}$ on our realizable witnesses.

\subsection{Moore--Penrose and zero-prime conventions}

Every pseudoinverse is the unique Moore--Penrose pseudoinverse.  Given a
diagonal prime matrix $D$, the second stage is exactly
\[
z=(XD)^\dagger y,\qquad w=Dz.
\]
If $p_i=0$, then $w_i=0$ automatically.  In the weighted-norm formulation,
we restrict the optimization to $\supp(p)$ and fix the remaining coordinates
to zero.  None of the proofs divides by a zero prime.

For the univariate rule, a historical zero column receives prime zero.  On
nonzero columns the one-dimensional pseudoinverse formula is used exactly.

\subsection{Pearson convention}

For a history of length $n\ge2$, write
\[
\widetilde X_i=X(:,i)-\bar X_i\one,
\qquad \widetilde y=y-\bar y\one.
\]
When both centered vectors are nonzero, the Pearson prime is
\[
p_i=\frac{\widetilde X_i^\top\widetilde y}
{\|\widetilde X_i\|_2\|\widetilde y\|_2}.
\]
The original statistic is undefined for histories of length zero or one and,
at longer histories, whenever either empirical variance is zero.  We call a
rule a \emph{finite totalization} of Pearson priming if it agrees with the
displayed correlation whenever that correlation is defined and assigns every
undefined prime an arbitrary finite real value determined only by the past.
In particular, the fill cannot depend on the current input, current label, or
future data.
The second-stage predictor is then computed exactly as in
Equation~\eqref{eq:primed-predictor}.

Both Pearson witnesses are independent of the chosen totalization.  The
rule-specific construction in Section~\ref{app:pearson} starts with one
opposite pair and charges only later query rounds, so its two
  convention-dependent seed losses are discarded.  The shared powered
  construction in Appendix~\ref{app:powered} charges the empty-history query,
  whose prediction is necessarily zero, and omits every recovery loss from its
  lower-bound sum; only the first recovery can depend on the Pearson fill.
  Every later charged history has positive variance in all active features and
  labels.  Coordinate padding may add zero-variance columns, but their filled
  primes multiply identically zero features and cannot change a charged
  prediction.

\subsection{Raw and clipped predictions}

The literal predictor uses $\widehat y=x^\top w$.  Since all labels in the
lower bound are in $[-1,1]$, the clipped value
$\clip(\widehat y)$ never has larger square loss.  Theorem~\ref{thm:main}
proves its lower bound directly after clipping; it does not infer clipped
failure from an out-of-range raw prediction.

\subsection{Comparator and boundedness}

In every family, $u=e_1$ is fixed before the sequence and
$y_t=x_{t,1}$.  Hence its cumulative loss is identically zero.  The Hadamard
families are sign-valued in their base dimensions; the arbitrary-dimensional
embedding pads zero coordinates and therefore lies in $\{-1,0,1\}^d$.  The
triangular family lies in $[-1,1]^d$ and has rational entries.

  The constraint $x_t\in[-1,1]^d$ is coordinatewise, so the sign-valued
  Hadamard rows have Euclidean norm $\sqrt d$.  Uniformly rescaling inputs and
  labels changes the scale at which clipping acts, so the clipped Hadamard
  lower bounds do not automatically transfer to Euclidean-normalized inputs.
  The present theorem makes no such normalized-input claim.

%% file: appendix/B_common_tools.tex
\section{Common algebraic and geometric tools}
\label{app:core}

\subsection{Minimum-norm identities used throughout}

\begin{lemma}[Weighted minimum-norm representation]
\label{lem:weighted-min-norm}
Let $p\in\R^d$, $D=\diag(p)$, and suppose
$y\in\operatorname{col}(XD)$.  Then $z=(XD)^\dagger y$ uniquely solves
\[
\min_{z\in\R^d}\|z\|_2^2
\quad\text{subject to}\quad XDz=y.
\]
Writing $w=Dz$ gives the equivalent problem
\[
\min_w \sum_{i:p_i\ne0}\frac{w_i^2}{p_i^2}
\quad\text{subject to}\quad Xw=y,
\qquad w_i=0\ \text{if }p_i=0.
\]
This is the prime-weighted geometry used in
Lemma~\ref{lem:target-mass}; zero-prime coordinates are fixed at zero and
never appear in a denominator.
\end{lemma}

\begin{proof}
Because $y\in\operatorname{col}(XD)$, the Moore--Penrose vector
$(XD)^\dagger y$ is the unique minimum-Euclidean-norm solution of the
consistent system $XDz=y$.  If $z$ is feasible, setting $z_i=0$ off
$\supp(p)$ does not change $XDz$ and can only decrease its norm.  Hence every
minimum-norm feasible vector is supported on $\supp(p)$.  Restricted to this
subspace, the map $w=Dz$ satisfies $Xw=y$, $w_i=0$ whenever $p_i=0$, and
$z_i=w_i/p_i$ on $\supp(p)$.  Conversely, if $Xw=y$ and $w_i=0$ off
$\supp(p)$, the vector defined by $z_i=w_i/p_i$ on $\supp(p)$ and zero
elsewhere is feasible for $XDz=y$.  On these restricted spaces the two maps
are inverse to one another.  Moreover,
\[
\|z\|_2^2
=\sum_{i:p_i\ne0}z_i^2
=\sum_{i:p_i\ne0}\frac{w_i^2}{p_i^2}.
\]
Thus the bijection preserves the objective and identifies the two
optimization problems and their unique minimum-norm representative.
\end{proof}

The next elementary facts are used when a construction is paired, padded, or
globally rescaled.  We state them explicitly because all three operations
involve rank-deficient matrices in some histories.

\begin{lemma}[Pseudoinverse invariances]
\label{lem:pseudoinverse-invariances}
Let $A$ be any real matrix.
\begin{enumerate}
  \item For every nonzero scalar $c$, $(cA)^\dagger=c^{-1}A^\dagger$.
  \item If $Q$ and $R$ are orthogonal matrices of compatible sizes, then
  \[
  (QA)^\dagger=A^\dagger Q^\top,
  \qquad
  (AR)^\dagger=R^\top A^\dagger.
  \]
  In particular, a simultaneous row permutation of a design and its label
  vector does not change the minimum-norm coefficient.
  \item If $\bar A=[A; -A]$ and $\bar y=[y;-y]$, then
  $\bar A^\dagger\bar y=A^\dagger y$.
  \item For every diagonal $D$ and nonzero scalar $c$,
  \[
  cD\bigl(X(cD)\bigr)^\dagger y=D(XD)^\dagger y.
  \]
\end{enumerate}
\end{lemma}

\begin{proof}
  We verify the identities without a rank assumption.  For a nonzero scalar
  $c$, put $B=c^{-1}A^\dagger$.  The four Moore--Penrose equations for the pair
  $(cA,B)$ are
  \begin{align*}
  (cA)B(cA)
    &=cAA^\dagger A
     =cA,\\
  B(cA)B
    &=c^{-1}A^\dagger AA^\dagger
     =B,\\
  \bigl((cA)B\bigr)^\top
    &=(AA^\dagger)^\top
     =AA^\dagger
     =(cA)B,\\
  \bigl(B(cA)\bigr)^\top
    &=(A^\dagger A)^\top
     =A^\dagger A
     =B(cA).
  \end{align*}
  Uniqueness of the Moore--Penrose inverse gives
  $(cA)^\dagger=c^{-1}A^\dagger$, including when $c<0$.

  Next take a full singular-value decomposition $A=U\Sigma V^\top$.  Then
  $A^\dagger=V\Sigma^\dagger U^\top$.  Orthogonality of $Q$ and $R$ gives
  \[
  QA=(QU)\Sigma V^\top,
  \qquad
  AR=U\Sigma(R^\top V)^\top,
  \]
  and hence
  \[
  (QA)^\dagger=V\Sigma^\dagger U^\top Q^\top=A^\dagger Q^\top,
  \qquad
  (AR)^\dagger=R^\top V\Sigma^\dagger U^\top=R^\top A^\dagger.
  \]
  Taking $Q$ to be a permutation matrix therefore yields
  \[
  (QA)^\dagger(Qy)
  =A^\dagger Q^\top Qy
  =A^\dagger y.
  \]

For the duplication identity, the least-squares objective for
$\bar A z\approx\bar y$ is
\[
\|Az-y\|_2^2+\|-Az+y\|_2^2=2\|Az-y\|_2^2.
\]
Thus the two problems have the same least-squares minimizers, and their unique
minimum-norm minimizers coincide.  The final identity is the first one with
$A=XD$:
\[
cD(cXD)^\dagger y
=cD\bigl(c^{-1}(XD)^\dagger\bigr)y
=D(XD)^\dagger y.
\]
\end{proof}

\begin{lemma}[Zero-coordinate embedding]
\label{lem:zero-coordinate-padding}
Append $q\ge0$ identically zero feature columns to every historical and
current input.  For the univariate and multivariate prime rules, at every
fixed power $\alpha>0$, the prediction on the original coordinates is
unchanged and all appended second-stage coefficients are zero.  The same is
true for a finitely totalized Pearson rule whenever its primes on the original
coordinates agree before and after padding; in particular, this holds when
all those original-coordinate correlations are defined.
\end{lemma}

\begin{proof}
For univariate priming, each appended historical column has prime zero by
definition.  For multivariate priming,
\[
[X,0]^\dagger y=\begin{bmatrix}X^\dagger y\\0_q\end{bmatrix},
\]
as follows either from an SVD or from the minimum-norm characterization: the
new coordinates do not affect the residual and are therefore zero in the
minimum-norm solution.  Thus the powered transformed design has the form
$[XD,0]$ in both cases.

For Pearson, assume the original primes are unchanged.  Each new prime may be
an arbitrary finite fill, but it multiplies an identically zero column, so the
powered transformed design again has the form $[XD,0]$.  This assumption is
automatic when the original correlations are defined, because appending
feature coordinates does not change their empirical means, variances, or
covariances.  In every stated case, the Moore--Penrose solution for
$[XD,0]$ is
\[
[XD,0]^\dagger y
=\begin{bmatrix}(XD)^\dagger y\\0_q\end{bmatrix}.
\]
Thus transforming back leaves the original coefficient and prediction
unchanged.  Without the agreement condition, a
dimension-dependent Pearson totalization may change its undefined fills on
old coordinates, and no general padding invariance is claimed.
\end{proof}

\begin{lemma}[Zero-tail extension]
\label{lem:zero-tail-extension}
After any fixed finite active sequence, appending rounds with $x_t=0$ and
$y_t=0$ leaves all earlier losses unchanged and contributes zero loss on every
new round, for both raw and clipped predictions.
\end{lemma}

\begin{proof}
Earlier predictions depend only on their completed prefixes and hence cannot
be changed by later data.  On every appended round, linearity gives
\[
x_t^\top w_t-y_t=0^\top w_t-0=0
\]
for any finite coefficient $w_t$.
\end{proof}

\subsection{Hadamard interpolation}

\begin{lemma}[Hadamard interpolation and dyadic rounding]
\label{lem:hadamard-rounding}
For every power of two $N$, the Sylvester construction provides
$H\in\{\pm1\}^{N\times N}$ with $HH^\top=NI_N$.  Let $H_r$ contain its first
$r$ rows.  Then $H_rH_r^\top=NI_r$, and for every $a\in\R^r$,
\begin{equation}
v=\frac1N H_r^\top a
\label{eq:nuisance-interpolant}
\end{equation}
satisfies $H_rv=a$.  This explicit nuisance-only interpolant drives the
univariate and Pearson target-mass bounds.

For every real $q\ge1$, some power of two $N$ satisfies $q\le N<2q$.
Taking $q=4(T-1)^2$ or $q=4M^2$ gives the dimensions below.
\end{lemma}

\begin{proof}
Starting from $H_1=[1]$, use the Sylvester recursion
\[
H_{2n}
=\begin{bmatrix}H_n&H_n\\H_n&-H_n\end{bmatrix}.
\]
If $H_nH_n^\top=nI_n$, block multiplication gives
\[
H_{2n}H_{2n}^\top
=\begin{bmatrix}
2H_nH_n^\top&0\\
0&2H_nH_n^\top
\end{bmatrix}
=2nI_{2n}.
\]
Restricting this identity to the first $r$ rows yields
$H_rH_r^\top=NI_r$.  Consequently, the vector in
Equation~\eqref{eq:nuisance-interpolant} satisfies
\[
H_rv
=\frac1N H_rH_r^\top a
=a.
\]
For the rounding statement take
$N=2^{\lceil\log_2q\rceil}$; then $q\le N<2q$, including when $q$ itself is
a power of two.
\end{proof}

\begin{lemma}[Minimum norm for collinear transformed columns]
\label{lem:collinear-min-norm}
Let the nonzero columns of a consistent design be $c_i u$, where
$u\ne0$ and the desired label vector is $u$.  Among all coefficients $z$
satisfying $\sum_i c_i z_i=1$, the unique minimum-norm vector is
\[
z_i=\frac{c_i}{\sum_jc_j^2},
\qquad
\|z\|_2^2=\frac1{\sum_jc_j^2}.
\]
\end{lemma}

\begin{proof}
  Cauchy--Schwarz gives
  \[
  1
  =\left(\sum_i c_i z_i\right)^2
  \le\left(\sum_i c_i^2\right)\left(\sum_i z_i^2\right).
  \]
  Equality holds exactly when $z=\lambda c$ for some scalar $\lambda$.
  Substituting this form into the constraint gives
  \[
  1=\lambda\sum_i c_i^2,
  \qquad
  \lambda=\frac1{\sum_i c_i^2}.
  \]
  This proves both the displayed coefficient and its squared norm.
\end{proof}

\begin{lemma}[Exact interpolation criterion for a prime refit]
\label{lem:prime-refit-interpolation}
Let $D$ be any diagonal matrix, set $A=XD$, and let
$w=DA^\dagger y$.  Then
\[
Xw=AA^\dagger y=P_{\operatorname{col}(A)}y.
\]
Consequently the second-stage coefficient interpolates the history if and
only if $y\in\operatorname{col}(XD)$.
\end{lemma}

\begin{proof}
Substituting $w=DA^\dagger y$ and $A=XD$ gives
\[
Xw
=XD A^\dagger y
=AA^\dagger y.
\]
Take a compact singular-value
decomposition $A=U_r\Sigma_rV_r^\top$.  Then
$A^\dagger=V_r\Sigma_r^{-1}U_r^\top$, so
\[
AA^\dagger
=U_r\Sigma_rV_r^\top V_r\Sigma_r^{-1}U_r^\top
=U_rU_r^\top,
\]
the orthogonal projector onto
$\operatorname{span}(U_r)=\operatorname{col}(A)$.  It fixes $y$ exactly if
and only if $y\in\operatorname{col}(A)$.
\end{proof}

%% file: appendix/C_unit_power_intro.tex
\section{Proof of the unit-power sparse-regret separation}
\label{app:unit-power}

This appendix proves Theorem~\ref{thm:main} in the order in which its
construction is assembled.  We first state finite-block bounds for the three
rules, then prove them from one Hadamard geometry, and finally convert the
finite blocks to arbitrary horizons and dimensions.

\subsection{Rule-specific finite-block bounds}

\begin{proposition}[Clipping-robust rule-specific bounds]
\label{prop:unit-power-constants}
Consider the past-only protocol in Section~\ref{sec:setup} with comparator
$u=e_1$.
\begin{enumerate}
  \item For every $T\ge2$, there exists a power of two $N$ such that
  \[
  4(T-1)^2\le N<8(T-1)^2,
  \qquad d=N+1.
  \]
  For each of the clipped multivariate LLS-prime and univariate-prime
  predictors, there is a rule-dependent offline-constructed sequence, fixed
  before play, $x_1,\ldots,x_T\in\{\pm1\}^d$ with $y_t=x_{t,1}$ such that
  \[
  \Reg_T^{\mathrm{clip},\LLS}\ge \frac{T}{4},
  \qquad
  \Reg_T^{\mathrm{clip},\UNI}\ge \frac{9T}{16},
  \]
  respectively.
  \item For every $M\ge1$, there exists a power of two $N$ satisfying
  $4M^2\le N<8M^2$ and set $d=N+1$, $T=2+2M$.  Under every finite Pearson
  totalization from Section~\ref{sec:setup}, one offline-constructed
  sign-valued sequence, independent of the totalization and satisfying
  $y_t=x_{t,1}$, obeys
  \[
  \Reg_T^{\mathrm{clip},\PEAR}\ge\frac{9M}{16}
  =\frac9{32}(T-2).
  \]
  Every charged round has positive empirical variance in every feature and
  in the labels.
\end{enumerate}
In all three cases the one-sparse comparator has zero cumulative loss.
\end{proposition}

%% file: appendix/C1_lls_hadamard.tex
\subsection{Multivariate LLS prime}
\label{app:lls}

This section gives the complete multivariate part of
Proposition~\ref{prop:unit-power-constants}.  The proof first computes the
prime on an arbitrary Hadamard prefix, then turns the resulting target-mass
bound into a recursively fixed sign sequence.

\subsubsection{The prime and its target mass}

Let $N$ be a power of two, $d=N+1$, and let
$h_1,\ldots,h_N$ be distinct rows of an $N\times N$ Hadamard matrix.  After
$r\ge1$ rounds, write
\[
X=[a,H_r],
\qquad
a=(a_1,\ldots,a_r)^\top\in\{\pm1\}^r,
\qquad
b=H_r^\top a.
\]
The target column is $a$, and the nuisance block is $H_r$.

\begin{lemma}[Multivariate prime on a Hadamard prefix]
\label{lem:lls-hadamard-prefix}
For the history above, the multivariate first-stage prime is
\begin{equation}
p=X^\dagger a
=\left(\frac r{N+r},\frac b{N+r}\right).
\label{eq:lls-full-prime}
\end{equation}
If $w=\diag(p)(X\diag(p))^\dagger a$ is the second-stage coefficient, then
\begin{equation}
0\le w_1\le (N+1)\frac{r^2}{(N+r)^2}.
\label{eq:lls-prefix-target-bound}
\end{equation}
In particular, $w_1\le1/2$ whenever $r\le\sqrt N/2$.
\end{lemma}

\begin{proof}
Hadamard orthogonality gives
\[
XX^\top=aa^\top+H_rH_r^\top=aa^\top+NI_r.
\]
This matrix is positive definite, so $X$ has full row rank and
$X^\dagger=X^\top(XX^\top)^{-1}$.  Since $a^\top a=r$,
\[
(NI_r+aa^\top)a=(N+r)a,
\qquad
(NI_r+aa^\top)^{-1}a=\frac a{N+r}.
\]
Multiplying by $X^\top=[a,H_r]^\top$ proves
Equation~\eqref{eq:lls-full-prime}; in particular,
\[
p_1=\frac r{N+r}.
\]

Because $X$ has full row rank,
\[
Xp=XX^\dagger a=a.
\]
Let $q_i=1$ on $\supp(p)$ and $q_i=0$ otherwise.  For $D=\diag(p)$,
\[
Dq=p,
\qquad
XDq=Xp=a.
\]
The vector $q$ is therefore feasible for the transformed
interpolation problem.  If
$J=\|(XD)^\dagger a\|_2^2$, Lemma~\ref{lem:weighted-min-norm} gives
\[
J\le\|q\|_2^2=|\supp(p)|\le d=N+1.
\]
The target-mass identity now yields
\[
0\le w_1=p_1^2J
\le(N+1)\frac{r^2}{(N+r)^2},
\]
which is Equation~\eqref{eq:lls-prefix-target-bound}.  If
$r\le\sqrt N/2$, then
\[
w_1\le\frac{(N+1)r^2}{N^2}
\le\frac{N+1}{4N}\le\frac12,
\]
where the last inequality uses $N\ge1$.
\end{proof}

\subsubsection{Offline sign compilation}

\begin{proof}[Proof of Proposition~\ref{prop:unit-power-constants},
multivariate case]
For a completed prefix $a_1,\ldots,a_r$, let $w^{(r)}$ denote the
coefficient computed from that history.  The online protocol fixes
$w^{(r)}$ before the next input is shown.  For each candidate sign
$\sigma\in\{\pm1\}$, consider
\[
x_{r+1}^{\sigma}=(\sigma,h_{r+1}),
\qquad y_{r+1}^{\sigma}=\sigma.
\]
Writing $c_r=h_{r+1}^\top w^{(r)}_{2:d}$, the two raw predictions are
\[
\widehat y_{r+1}^{\sigma}
=c_r+\sigma w_1^{(r)},
\qquad \sigma\in\{\pm1\}.
\]
If
$1\le r\le\sqrt N/2$, Lemma~\ref{lem:lls-hadamard-prefix} gives
$0\le w^{(r)}_1\le1/2$, and Lemma~\ref{lem:two-sign} gives
\[
\frac12\sum_{\sigma\in\{\pm1\}}
\left(\clip\!\left(c_r+\sigma w^{(r)}_1\right)-\sigma\right)^2
\ge(1-w^{(r)}_1)^2\ge\frac14.
\]
  Since the maximum of two numbers is at least their average,
  \[
  \max_{\sigma\in\{\pm1\}}
  \left(\clip(c_r+\sigma w_1^{(r)})-\sigma\right)^2\ge\frac14.
  \]
  Append a maximizing sign, choosing $+1$ when the two losses tie.  At $r=0$
  the coefficient is zero, so both signs incur loss one and the same tie rule
  applies.

This recursion is finite and deterministic once the multivariate rule,
$N$, the Hadamard rows, and a tie-breaking convention are fixed.  It therefore
defines the full sign sequence before the online interaction is replayed; the
sequence does not react to runtime predictions.  Every label is the first
coordinate, so the fixed comparator $e_1$ has zero loss.

Let
$L=\lfloor\sqrt N/2\rfloor+1$.  Before each of the first $L$ rounds the
number of historical rows satisfies
$r\le L-1\le\sqrt N/2$.  The preceding argument therefore gives
\[
\Reg_L^{\mathrm{clip},\LLS}\ge\frac L4.
\]
The construction has enough distinct Hadamard rows: in the regime used by the
proposition $N\ge4$, and then $L\le N$.
If $N\ge4(T-1)^2$, then $T-1\le\sqrt N/2$, hence $T\le L$; truncating the
compiled sequence after $T$ rounds proves
\[
\Reg_T^{\mathrm{clip},\LLS}\ge \frac T4.
\]
\end{proof}

%% file: appendix/C2_univariate_hadamard.tex
\subsection{Univariate prime}
\label{app:univariate-hadamard}

\begin{lemma}[Univariate target mass on a Hadamard prefix]
\label{lem:univariate-hadamard-prefix}
Let $X=[a,H_r]$ with $a\in\{\pm1\}^r$ and $r\ge1$.  The univariate primes
satisfy
\begin{equation}
p_1=1,
\qquad
p_{j+1}=\frac{b_j}{r}\quad(j\in[N]),
\qquad b=H_r^\top a.
\label{eq:uni-prime}
\end{equation}
The corresponding second-stage target coefficient obeys
\[
0\le w_1\le\frac{r^2}{N}.
\]
\end{lemma}

\begin{proof}
The target historical column is exactly $a$, hence its one-dimensional
regression coefficient is one.  Each nuisance column has squared norm $r$,
which proves Equation~\eqref{eq:uni-prime}.  Let
$v=N^{-1}H_r^\top a=b/N$.  By
Equation~\eqref{eq:nuisance-interpolant}, $H_rv=a$.

We construct a feasible transformed coefficient with zero target coordinate.
Set $z_1^{\rm nui}=0$.  For each $j\in[N]$, when $b_j\ne0$, set
\[
z_{j+1}^{\rm nui}=\frac{v_j}{p_{j+1}}
=\frac{b_j/N}{b_j/r}=\frac rN;
\]
when $b_j=0$, both $v_j$ and $p_{j+1}$ are zero, so set
$z_{j+1}^{\rm nui}=0$.  This support convention is essential: no division by a
zero prime occurs.  Feasibility follows coordinate by coordinate:
\[
D z^{\rm nui}
=\begin{bmatrix}0\\v\end{bmatrix},
\qquad
XDz^{\rm nui}
=H_rv
=a.
\]
Moreover,
\[
\|z^{\rm nui}\|_2^2
=|\{j:b_j\ne0\}|\frac{r^2}{N^2}
\le\frac{r^2}{N}.
\]
By minimum-norm optimality the actual transformed objective $J$ is no larger.
Since $p_1=1$, Lemma~\ref{lem:target-mass} gives $w_1=J$, proving the claim.
\end{proof}

\begin{proof}[Proof of Proposition~\ref{prop:unit-power-constants},
univariate case]
After a completed prefix of $r$ labels, let $w^{(r)}$ be the univariate-primed
coefficient computed from that prefix.  It is fixed before the next input is
revealed.  Hold the next nuisance row $h_{r+1}$ fixed and write
$c_r=h_{r+1}^\top w^{(r)}_{2:d}$.  The two legal candidate examples are
\[
x_{r+1}^{\sigma}=(\sigma,h_{r+1}),
\qquad y_{r+1}^{\sigma}=\sigma,
\qquad \sigma\in\{\pm1\},
\]
Their raw predictions can be written uniformly as
\[
\widehat y_{r+1}^{\sigma}
=c_r+\sigma w_1^{(r)},
\qquad \sigma\in\{\pm1\}.
\]
When $1\le r\le\sqrt N/2$, Lemma
\ref{lem:univariate-hadamard-prefix} gives
$0\le w_1^{(r)}\le1/4$.  Lemma~\ref{lem:two-sign} therefore gives
\[
\frac12\sum_{\sigma\in\{\pm1\}}
\left(\clip(c_r+\sigma w_1^{(r)})-\sigma\right)^2
\ge(1-w_1^{(r)})^2\ge\frac9{16}.
\]
  The maximum of the two losses is at least their displayed average.  Append a
  maximizing sign, choosing $+1$ in a tie.  At $r=0$ the empty-history
  coefficient is zero, so both signs lose one and the same tie rule applies.
  Induction therefore compiles a finite deterministic
sign string before it is replayed online; no sign is selected from a runtime
prediction.

For $L=\lfloor\sqrt N/2\rfloor+1$, all $L$ rounds are therefore valid and
\[
\Reg_L^{\mathrm{clip},\UNI}
\ge1+\frac9{16}(L-1)
\ge\frac{9L}{16}.
\]
For $N\ge4$, one has $L\le N$, so all required Hadamard rows are available.
If $N\ge4(T-1)^2$, then $T\le L$; truncation after $T$ rounds gives
\[
\Reg_T^{\mathrm{clip},\UNI}\ge\frac{9T}{16}.
\]
\end{proof}

%% file: appendix/C3_pearson_hadamard.tex
\subsection{Pearson prime}
\label{app:pearson}

Pearson correlation is undefined on an empty or constant history.  We isolate
that issue in a two-round seed and thereafter maintain a paired state in which
every active correlation is defined.  This yields one sign sequence that is
simultaneously valid for every finite past-only totalization.

\subsubsection{The paired state}

For $s\ge1$, a \emph{completed paired history} consists, for each
$\ell\in[s]$, of the two examples
\[
((a_\ell,h_\ell),a_\ell)
\quad\text{and}\quad
((-a_\ell,-h_\ell),-a_\ell),
\qquad a_\ell\in\{\pm1\}.
\]
Let $a=(a_1,\ldots,a_s)^\top$, let $H_s$ contain the corresponding Hadamard
rows, and write $b=H_s^\top a$.

\begin{lemma}[Pearson primes and target mass in a paired state]
\label{lem:pearson-paired-state}
On a completed paired history, every feature column and the label vector has
positive empirical variance.  The Pearson primes are uniquely determined by
the usual correlation formula and satisfy
\begin{equation}
p_1=1,
\qquad
p_{j+1}=\frac{b_j}{s}\quad(j\in[N]).
\label{eq:pearson-prime}
\end{equation}
The second-stage target coefficient obeys
\[
0\le w_1\le\frac{s^2}{N}.
\]
Both conclusions are independent of the finite totalization.
\end{lemma}

\begin{proof}
Each paired scalar column has the form $(q_1,-q_1,\ldots,q_s,-q_s)$ with
$q_\ell\in\{\pm1\}$.  Its empirical mean is zero and its squared centered
norm is $2s>0$; the label vector has the same properties.  The target column
equals the label vector, so its correlation is one.  For nuisance coordinate
$j$, the centered covariance numerator is
\[
\sum_{\ell=1}^s
\bigl(h_{\ell j}a_\ell+(-h_{\ell j})(-a_\ell)\bigr)
=2H_s(:,j)^\top a=2b_j,
\]
while both centered norms equal $\sqrt{2s}$.  This proves
Equation~\eqref{eq:pearson-prime}.  Because every variance is positive, a
totalization has no freedom on these coordinates.

For completeness, we repeat rather than merely cite the nuisance certificate.
Let
\[
v=\frac1N H_s^\top a=\frac bN,
\qquad
H_sv=a.
\]
Give the transformed target coordinate value zero.  If $b_j\ne0$, set
\[
z_{j+1}^{\rm nui}=\frac{v_j}{p_{j+1}}
=\frac{b_j/N}{b_j/s}=\frac sN;
\]
if $b_j=0$, then both $v_j$ and $p_{j+1}$ vanish and we set
$z_{j+1}^{\rm nui}=0$.  On one representative from each pair the transformed
nuisance prediction is $H_sv=a$; by sign symmetry it also fits every opposite
  row because
\[
(-H_s)v=-a.
\]
Thus this vector is feasible for the full paired transformed system and
\[
\|z^{\rm nui}\|_2^2
=|\{j:b_j\ne0\}|\frac{s^2}{N^2}
\le\frac{s^2}{N}.
\]
Minimum-norm optimality and Lemma~\ref{lem:target-mass}, with $p_1=1$, give
$0\le w_1=J\le s^2/N$.
\end{proof}

\begin{lemma}[Query--recovery invariant]
\label{lem:pearson-query-recovery}
Assume the history consists of $s\ge1$ completed pairs and
$s\le\sqrt N/2$.  There is a query sign
$a_{s+1}\in\{\pm1\}$, independent of the finite Pearson totalization, whose
clipped loss is at least $9/16$.  After its label is revealed, presenting the
exact opposite example gives raw and clipped loss zero and restores a
completed paired history with $s+1$ pairs.
\end{lemma}

\begin{proof}
Before the query, Lemma~\ref{lem:pearson-paired-state} gives
$0\le w_1\le s^2/N\le1/4$.  The coefficient is computed before the current
input.  Holding the nuisance row $h_{s+1}$ fixed, the two legal inputs
$(+1,h_{s+1})$ and $(-1,h_{s+1})$ therefore satisfy the hypotheses of
  Lemma~\ref{lem:two-sign}.  The larger of the two losses is at least their
  average, so choose a maximizing sign, taking $+1$ in a tie.  Its loss is at
  least $(1-w_1)^2\ge9/16$.  Every prime used for this query is
uniquely defined by the paired history, so the chosen sign is the same for all
finite totalizations.

Now append the query and reveal its label.  In this enlarged history the
  target feature column still equals the complete label vector and is
  nonconstant, hence its Pearson prime is one.  If $D$ is the diagonal prime
  matrix, then
\[
XD e_1=Xe_1=y.
\]
Thus the transformed coefficient $e_1$ is feasible for $XDz=y$, so
$y\in\operatorname{col}(XD)$, and
  Lemma~\ref{lem:prime-refit-interpolation} shows that the Moore--Penrose refit
  interpolates every stored row.  In particular, its raw prediction on the just
  appended query $x$ equals its label $y$.  The next round uses this refitted
coefficient and input $-x$ with label $-y$; linearity gives
\[
(-x)^\top w=-x^\top w=-y.
\]
Thus the recovery loss is zero.  Appending the recovery
row restores zero means and completes the new opposite pair.
\end{proof}

\subsubsection{The totalization-uniform construction}

\begin{proof}[Proof of Proposition~\ref{prop:unit-power-constants}, Pearson
case]
Start with the fixed seed pair
\[
(x_1,y_1)=((1,h_1),1),
\qquad
(x_2,y_2)=(-(1,h_1),-1).
\]
The empty-history coefficient is zero.  The prediction on the second seed
round may depend on the arbitrary length-one totalization, but it is finite;
neither seed loss is used in the lower bound.  After the pair is stored, the
history is in the state of Lemma~\ref{lem:pearson-paired-state} with $s=1$.

For $s=1,\ldots,M$, apply Lemma~\ref{lem:pearson-query-recovery}, using a
fixed tie-breaking convention, and append the selected query and its recovery.
If $N\ge4M^2$, every charged stage satisfies $s\le\sqrt N/2$.  The construction
also has enough distinct nuisance rows: $M+1\le N$ follows from
$N\ge4M^2$ for $M\ge1$.  Because the paired-state prediction and the selected
sign are independent of the totalization at every stage, this single finite
recursion produces one sign sequence that works simultaneously for all finite
past-only Pearson totalizations.  The entire sequence is fixed before online
replay.

There are $M$ charged queries, each with loss at least $9/16$; all recovery
losses are zero and the two seed losses are nonnegative.  Since every label is
the first coordinate, comparator $e_1$ has zero loss.  With $T=2+2M$,
\[
\Reg_T^{\mathrm{clip},\PEAR}
\ge\frac{9M}{16}=\frac9{32}(T-2).
\]
This proves the Pearson part of Proposition
\ref{prop:unit-power-constants}.
\end{proof}

%% file: appendix/C4_unit_power_assembly.tex
\subsection{From finite blocks to arbitrary horizons and dimensions}

The rule-specific proofs in Sections~\ref{app:lls}--\ref{app:pearson}
establish a slightly more flexible statement than
Proposition~\ref{prop:unit-power-constants}.  For any power of two
$N\ge4$, the multivariate and univariate compilers are valid for every length
at most
\[
L_N=\left\lfloor\frac{\sqrt N}{2}\right\rfloor+1,
\]
with per-round lower bounds $1/4$ and $9/16$, respectively.  For $N\ge4$,
$L_N\le N$, so all rows $h_1,\ldots,h_{L_N}$ required by either compiler are
available.  The Pearson
compiler is valid for any number
$q\le\lfloor\sqrt N/2\rfloor$ of charged query--recovery cycles after its
seed pair.  Indeed, these are precisely the prefix conditions checked in
Lemmas~\ref{lem:lls-hadamard-prefix},
\ref{lem:univariate-hadamard-prefix}, and
\ref{lem:pearson-query-recovery}.  We now choose the largest available
Hadamard block and account explicitly for all rounding.

\begin{proof}[Proof of Theorem~\ref{thm:main} using the admissible-prefix
constructions]
  For the multivariate and univariate rules, take
  $x_1=(1,\ldots,1)\in\{\pm1\}^d$ and $y_1=1$ when $T=1$.  The empty-history
  transformed design has zero rows, so its Moore--Penrose coefficient and the
  prediction are zero; the clipped loss is therefore $(0-1)^2=1$.
Hence assume $T\ge2$ for those rules; the Pearson statement already assumes
$T\ge4$.

Let $N$ be the largest power of two not exceeding $d-1$ and set $d'=N+1$.
Then $N\le d-1<2N$.  Since $d\ge5$, we have $N\ge4$, and
\[
d'=N+1>\frac d2,
\qquad
d'=N+1\le\frac{5N}{4}.
\]
Run the construction in the first $d'$ coordinates and pad the remaining
$d-d'$ coordinates with zeros.  Lemma~\ref{lem:zero-coordinate-padding}
shows that this embedding preserves every active prediction for the
multivariate and univariate rules.  For Pearson, every charged query begins
from a completed paired history.  Lemma~\ref{lem:pearson-paired-state} shows
that all original-coordinate correlations are then defined and unchanged by
padding, while every padded prime multiplies an identically zero column.  The
same is true after a query when the recovery coefficient is fitted.  Thus the
charged queries and exact recoveries remain totalization-independent; only
the two discarded seed losses may change with the dimension-dependent fill.

For the first two rules, compile $L=L_N$ active rounds.  Truncate if
$T<L$, and if $T>L$ use the zero-tail extension in
Lemma~\ref{lem:zero-tail-extension}.  The smaller per-round constant is
$1/4$, so the regret is at least $\frac14\min\{T,L\}$.  Since
$L=\lfloor\sqrt N/2\rfloor+1\ge\sqrt N/2$ and
$d'\le5N/4$,
\[
L\ge\frac{\sqrt{d'}}{\sqrt5}.
\]
For every $0<c\le1$, $\min\{A,cB\}\ge c\min\{A,B\}$.  Applying this fact
twice, first with $c=1/\sqrt5$ and then using
$\sqrt{d'}>\sqrt d/\sqrt2$, gives
\begin{align*}
\Reg_T^{\rm clip}
&\ge\frac14\min\{T,L\}
\ge\frac1{4\sqrt5}\min\{T,\sqrt{d'}\}\\
&\ge\frac1{4\sqrt{10}}\min\{T,\sqrt d\}.
\end{align*}

For Pearson, use
\[
q=\min\!\left\{\left\lfloor\frac{T-2}{2}\right\rfloor,
                \left\lfloor\frac{\sqrt N}{2}\right\rfloor\right\}
\]
query--opposite cycles after the seed pair.  This choice satisfies
$2+2q\le T$ and the Pearson admissibility condition
$q\le\lfloor\sqrt N/2\rfloor$.  It also leaves enough distinct Hadamard
rows for the seed and all queries: since $N\ge4$,
$q+1\le\lfloor\sqrt N/2\rfloor+1\le N$.  For $T\ge4$,
\[
\left\lfloor\frac{T-2}{2}\right\rfloor\ge\frac T5:
\]
the case $T=4$ holds directly because $1\ge4/5$.  For $T\ge5$,
\[
\left\lfloor\frac{T-2}{2}\right\rfloor
\ge\frac{T-3}{2}
\ge\frac T5.
\]
Since $N\ge4$,
\[
\left\lfloor\frac{\sqrt N}{2}\right\rfloor
\ge\frac{\sqrt N}{4}
\ge\frac{\sqrt{d'}}{2\sqrt5}
\ge\frac{\sqrt{d'}}5.
\]
It follows that
$q\ge\frac15\min\{T,\sqrt{d'}\}$.  The $q$ query losses sum to at least
\[
\frac{9q}{16}
\ge\frac9{80}\min\{T,\sqrt{d'}\}
\ge\frac9{80\sqrt2}\min\{T,\sqrt d\}.
\]
If $2+2q<T$, append the remaining zero-input, zero-label tail; by
Lemma~\ref{lem:zero-tail-extension} it contributes exactly zero loss.

Finally,
\[
\frac1{4\sqrt{10}}>\frac1{15},
\qquad
\frac9{80\sqrt2}>\frac1{15},
\]
which proves the displayed constant in Theorem~\ref{thm:main}.  To prove its
sparse-logarithmic consequence, fix any of the three rules and suppose that
some finite constant $C$ gave the uniform bound
\[
\Reg_T\le Ck\log\!\left(\frac{ed}{k}\right).
\]
Take $k=1$ and dimensions of order $d=T^2$, rounding to the next admissible
integer if needed.  Along this sequence of horizons,
\[
\Reg_T\ge \frac{T}{15}
\qquad\text{but}\qquad
C\log(ed)=O(C\log T).
\]
This is impossible as $T\to\infty$.  Because clipping
cannot increase square loss, the same sequence also refutes the corresponding
raw-prediction guarantee.
\end{proof}

%% file: appendix/D_powered_shared.tex
\section{Proof of the powered shared-witness theorem}
\label{app:powered}

Fix \(\alpha\ge1\) and define
\[
f_\alpha(q)=\operatorname{sign}(q)|q|^\alpha,
\qquad
D_\alpha(p)=\diag\bigl(f_\alpha(p_1),\ldots,f_\alpha(p_d)\bigr).
\]
The powered rule replaces the base-prime matrix by \(D_\alpha(p)\) in the
second Moore--Penrose fit.  We first record the individual target-mass bounds,
then prove the stronger fact used by Theorem~\ref{thm:powered}: on completed
paired histories, all three powered refits are exactly the same vector.

\subsection{Individual powered target-mass certificates}

\begin{lemma}[Powered Hadamard certificates]
\label{lem:powered-individual-certificates}
Let \(H_r\) contain \(r\ge1\) rows of an \(N\times N\) Hadamard matrix,
let \(a\in\{\pm1\}^r\), and put \(X=[a,H_r]\).  Before the next query, the
target coefficient of the powered-univariate refit satisfies
\[
0\le w_1^{\UNI}\le \frac{r^{2\alpha}}N,
\]
whereas the powered-multivariate refit satisfies
\[
0\le w_1^{\LLS}\le \frac{r^{2\alpha}}N.
\]
After \(s\ge1\) completed opposite pairs, powered Pearson priming satisfies
the univariate bound with \(r=s\), independently of every finite Pearson
totalization.
\end{lemma}

\begin{proof}
Write \(b=H_r^\top a\).  For univariate priming, the target base prime is
one and the nuisance base primes are
\[
p_1=1,
\qquad
p_{j+1}=\frac{b_j}{r},\quad j\in[N].
\]
The nuisance vector \(v=N^{-1}H_r^\top a=b/N\) satisfies \(H_rv=a\) by
Lemma~\ref{lem:hadamard-rounding}.  For \(b_j\ne0\), set
\[
z_{j+1}=\frac{v_j}{f_\alpha(b_j/r)},
\]
and set \(z_{j+1}=0\) when \(b_j=0\); also set \(z_1=0\).  If \(b_j=0\),
then \(v_j=0\), so no interpolation mass is lost.  If \(b_j\ne0\), its
integer value has magnitude at least one and
\[
|z_{j+1}|
=\frac{|b_j|/N}{(|b_j|/r)^\alpha}
=\frac{r^\alpha}{N|b_j|^{\alpha-1}}
\le \frac{r^\alpha}{N}.
\]
Coordinatewise, \(D_\alpha(p)z=(0,v)\), and hence
\[
XD_\alpha(p)z=H_rv=a,
\qquad
\|z\|_2^2
\le N\left(\frac{r^\alpha}{N}\right)^2
=\frac{r^{2\alpha}}N.
\]
The
target-mass identity, Lemma~\ref{lem:target-mass}, now gives
\(0\le w_1^{\UNI}\le r^{2\alpha}/N\).

For multivariate priming, Lemma~\ref{lem:lls-hadamard-prefix} gives the full
base-prime vector
\[
p_1=\frac r{N+r},
\qquad
p_{j+1}=\frac{b_j}{N+r},\quad j\in[N],
\qquad Xp=a.
\]
Use the nuisance interpolant \(v=b/N\) from the univariate calculation.
For \(b_j\ne0\), represent \(v_j\) in the transformed nuisance design by
\[
z_{j+1}=\frac{v_j}{f_\alpha(b_j/(N+r))},
\]
and set \(z_{j+1}=0\) when \(b_j=0\), as well as \(z_1=0\).  Since every
nonzero \(b_j\) is an integer,
\[
|z_{j+1}|
=\frac{(N+r)^\alpha}{N|b_j|^{\alpha-1}}
\le\frac{(N+r)^\alpha}{N}.
\]
This nuisance-only vector interpolates \(a\) and has squared norm at most
\((N+r)^{2\alpha}/N\).  If \(J\) is the minimum squared transformed norm,
Lemma~\ref{lem:target-mass} and feasibility give
\[
w_1^{\LLS}
=f_\alpha(p_1)^2J
\le \left(\frac{r}{N+r}\right)^{2\alpha}
     \frac{(N+r)^{2\alpha}}{N}
=\frac{r^{2\alpha}}N.
\]

Finally, after \(s\) completed opposite pairs, every active column and the
label vector have zero mean and squared centered norm \(2s\).  The Pearson
base primes are therefore
\[
p_1=1,
\qquad
p_{j+1}=\frac{H_s(:,j)^\top a}{s},
\quad j\in[N],
\]
exactly as in
  Lemma~\ref{lem:pearson-paired-state}.  On one representative from each pair,
  the powered transformed system is therefore the powered-univariate system
  with \(r=s\).  The full paired design and label vector stack this system with
  its negative.  By the opposite-duplication identity in
  Lemma~\ref{lem:pseudoinverse-invariances}, stacking leaves the
  Moore--Penrose coefficient unchanged.  The univariate certificate thus gives
  \(w_1^{\PEAR}\le s^{2\alpha}/N\).  Since all correlations used here are
  defined, the calculation is independent of the Pearson totalization.
\end{proof}

The individual bounds already allow rule-dependent adverse signs.  The shared
result is stronger because it identifies the entire coefficient vector, not
only its first coordinate.

\subsection{Exact equivalence on paired histories}

For \(s\ge1\), let \(H_s\) be the first \(s\) Hadamard rows and let
\(a=(\sigma_1,\ldots,\sigma_s)^\top\in\{\pm1\}^s\).  One representative
from each completed pair forms
\[
X_s=[a,H_s],\qquad y_s=a.
\]
Up to a row permutation, the full paired history is
\begin{equation}
\bar X_s=\begin{bmatrix}X_s\\-X_s\end{bmatrix},
\qquad
\bar y_s=\begin{bmatrix}a\\-a\end{bmatrix}.
\label{eq:paired-history}
\end{equation}

\begin{lemma}[Base primes on a paired history]
\label{lem:paired-base-primes}
Let \(b=H_s^\top a\).  On the history in
Equation~\eqref{eq:paired-history},
\[
p^{\UNI}=p^{\PEAR}=\left(1,\frac bs\right),
\qquad
p^{\LLS}=\left(\frac{s}{N+s},\frac b{N+s}\right)
=\frac{s}{N+s}p^{\UNI}.
\]
Every Pearson correlation in this display is defined.
\end{lemma}

\begin{proof}
All three first-stage statistics, and the second-stage Moore--Penrose refit,
are invariant under a common permutation of the historical rows.  For both
least-squares fits this follows from the orthogonal row-factor identity in
Lemma~\ref{lem:pseudoinverse-invariances}; for the univariate and Pearson
primes it also follows directly because the relevant inner products, means,
and norms are unchanged.  We may therefore work with the stacked order in
Equation~\eqref{eq:paired-history} throughout this subsection.

Duplicating every row with its negative doubles each univariate numerator and
denominator, so it leaves the univariate prime unchanged.  This gives
\(p_1^{\UNI}=1\) and \(p_{j+1}^{\UNI}=b_j/s\).

Every paired feature column and the paired label vector have mean zero and
squared norm \(2s>0\).  Their Pearson correlation therefore has denominator
\(2s\).  The target numerator is \(2s\), whereas nuisance coordinate \(j\)
has numerator \(2b_j\).  Hence \(p^{\PEAR}=p^{\UNI}\), including when
\(b_j=0\).

By Lemma~\ref{lem:pseudoinverse-invariances}, the first-stage multivariate
minimum-norm coefficient on \((\bar X_s,\bar y_s)\) equals that on
\((X_s,a)\).  Since \(X_sX_s^\top=NI_s+aa^\top\), the computation in
Lemma~\ref{lem:lls-hadamard-prefix} gives
\(p^{\LLS}=(s,b)/(N+s)\), proving the final identity.
\end{proof}

\begin{lemma}[Exact powered-refit equivalence]
\label{lem:paired-powered-equivalence}
Fix one common \(\alpha\ge1\).  On every nonempty completed paired history,
the three powered second-stage coefficient vectors are identical.  At empty
history, all three coefficient vectors are zero for every finite Pearson
totalization.
\end{lemma}

\begin{proof}
Let \(c_s=s/(N+s)>0\).  Positive homogeneity of \(f_\alpha\) and
Lemma~\ref{lem:paired-base-primes} give
\[
D_{\PEAR}=D_{\UNI},
\qquad
D_{\LLS}=c_s^\alpha D_{\UNI}.
\]
This remains true at zero-prime coordinates.  Global diagonal scaling cancels
from the refit by Lemma~\ref{lem:pseudoinverse-invariances}, so
\begin{equation}
D_{\LLS}(\bar X_sD_{\LLS})^\dagger\bar y_s
=D_{\UNI}(\bar X_sD_{\UNI})^\dagger\bar y_s
=D_{\PEAR}(\bar X_sD_{\PEAR})^\dagger\bar y_s.
\label{eq:paired-weight-equivalence}
\end{equation}
At empty history the transformed design has zero rows.  Its pseudoinverse
maps the empty label vector to zero, whatever finite primes a totalization
assigns, so all three returned coefficients are zero.
\end{proof}

\begin{lemma}[Shared target-mass and two-sign loss]
\label{lem:shared-powered-query}
Let \(w_s\) denote the common coefficient before the next query after \(s\)
completed pairs.  If \(0\le s<M\) and \(N\ge8M^{2\alpha}\), then
\[
0\le w_{s,1}\le\frac18.
\]
For the fixed next nuisance row \(h_{s+1}\), let \(L_s(\sigma)\) be the
clipped loss on current input \((\sigma,h_{s+1})\) and label \(\sigma\).
Then
\begin{equation}
\frac{L_s(+1)+L_s(-1)}2
\ge (1-w_{s,1})^2
\ge\frac{49}{64}.
\label{eq:shared-query-loss}
\end{equation}
\end{lemma}

\begin{proof}
For \(s=0\), \(w_s=0\).  For \(s\ge1\), apply the univariate part of
Lemma~\ref{lem:powered-individual-certificates} to one representative from
each pair.  Lemma~\ref{lem:pseudoinverse-invariances} shows that duplicating
the transformed equations with their negatives does not change the refit,
and therefore
\[
0\le w_{s,1}\le \frac{s^{2\alpha}}N
\le\frac{M^{2\alpha}}N\le\frac18.
\]
The common coefficient is computed before the current input is revealed.
Writing \(b_s=h_{s+1}^\top w_{s,2:d}\), the two raw predictions are
\[
\widehat y_s^{\sigma}
=b_s+\sigma w_{s,1},
\qquad \sigma\in\{\pm1\}.
\]
Lemma~\ref{lem:two-sign} gives the
first inequality in Equation~\eqref{eq:shared-query-loss}; the target-mass
bound gives the second.
\end{proof}

\subsection{Oblivious compilation and selector robustness}

\begin{proposition}[One compiled string for all three-rule selectors]
\label{prop:shared-powered-compiler}
Fix \(\alpha\ge1\) and \(M\ge1\), choose a power of two
\(N\in[8M^{2\alpha},16M^{2\alpha})\), and fix an \(N\times N\) Hadamard
matrix.  There exists a sign string
\(\sigma^*=(\sigma_1^*,\ldots,\sigma_M^*)\), depending only on these fixed
objects, such that the length-\(2M\) sequence
\begin{align*}
x_{2s+1}&=(\sigma_{s+1}^*,h_{s+1}),
&y_{2s+1}&=\sigma_{s+1}^*,\\
x_{2s+2}&=(-\sigma_{s+1}^*,-h_{s+1}),
&y_{2s+2}&=-\sigma_{s+1}^*,
\end{align*}
for \(s=0,\ldots,M-1\), gives clipped regret at least \(49M/64\) to
every three-rule selector in Definition~\ref{def:three-rule-selector}.  The
bound holds for every realization of the selector's random tape and every
finite Pearson totalization, under both mixture semantics
\eqref{eq:coefficient-selector} and \eqref{eq:prediction-selector}.
\end{proposition}

\begin{proof}
We construct the string by a finite greedy recursion.  Suppose the first
\(s\) signs have been fixed and their query--opposite pairs have been
appended.  This completed prefix determines the common vector \(w_s\) in
Lemma~\ref{lem:paired-powered-equivalence}.  Evaluate the two legal extensions
\(\sigma=+1\) and \(\sigma=-1\) on the fixed next nuisance row
  \(h_{s+1}\).  By Lemma~\ref{lem:shared-powered-query},
\[
\max_{\sigma\in\{\pm1\}}L_s(\sigma)
\ge\frac{L_s(+1)+L_s(-1)}2
\ge\frac{49}{64}.
\]
Select a
  maximizing sign, choosing \(+1\) in a tie, and
then append its exact opposite.  Induction for \(s=0,\ldots,M-1\) produces a
single finite string before online play.

On every charged query, Equation~\eqref{eq:paired-weight-equivalence} makes
the three powered second-stage coefficient vectors identical.  Hence any
convex coefficient
mixture equals \(w_s\), even if its weights depend on the current input and
the random tape.  The three raw scalar predictions, and therefore their
individually clipped versions, are also identical, so a convex mixture under
the second semantics gives the same charged prediction.  Switchers are
simplex vertices and are covered as well.  Thus every selector and every
random-tape realization incurs at least \(49/64\) on each of the \(M\)
queries.  We discard the nonnegative recovery losses.

At \(s=0\), the query uses the totalization-independent zero vector; the
first recovery may depend on the Pearson fill, but its loss is discarded.
For every \(s\ge1\), Lemma~\ref{lem:paired-base-primes} shows that all active
Pearson correlations are defined and independent of the fill.  The greedy
choice is therefore uniform over all finite Pearson totalizations.
\end{proof}

\begin{corollary}[Oblivious random-sign companion]
\label{cor:shared-random-signs}
Under the same fixed \((\alpha,M,N,H)\), let the query signs be independent
uniform Rademacher variables, independent also of the selector's random tape
\(\xi\), and append the exact opposite after every query.
For every three-rule selector and every fixed random-tape realization \(\xi\),
\[
\mathbb E_\sigma\!\left[\Reg_{2M}^{\rm clip}\mid\xi\right]
\ge\frac{49M}{64}.
\]
Consequently the same inequality holds after also averaging over any law of
\(\xi\).
\end{corollary}

\begin{proof}
Condition on the completed paired history before query \(s+1\) and on
\(\xi\).  Although the selector may see the current sign through the current
input, Lemma~\ref{lem:paired-powered-equivalence} makes all three available
outputs identical for either sign.  Its conditional expected query loss is
therefore the two-sign average in
Equation~\eqref{eq:shared-query-loss}, at least \(49/64\).  Summing these
conditional bounds over the \(M\) queries and discarding recovery losses proves
the first statement.  Averaging over \(\xi\) proves the second.
\end{proof}

\begin{proof}[Proof of Theorem~\ref{thm:powered}]
Apply Lemma~\ref{lem:hadamard-rounding} with
\(q=8M^{2\alpha}\) to obtain a power of two
\(N\in[8M^{2\alpha},16M^{2\alpha})\).  Proposition
\ref{prop:shared-powered-compiler} supplies one deterministic sequence for
all three rules and selectors.  The required Hadamard rows exist because
\(N\ge8M^{2\alpha}\ge M\) for \(M\ge1\) and \(\alpha\ge1\).  Every row is
sign-valued, every label equals
its first coordinate, and hence the fixed one-sparse comparator \(e_1\) has
zero loss.  Learner loss is therefore regret, and the proposition yields
\(\Reg_{2M}^{\rm clip}\ge49M/64\) pathwise.

With \(T=2M\) and \(d=N+1\),
\[
d=\Theta_\alpha(T^{2\alpha}),
\qquad
\log d=O_\alpha(\log T).
\]
Thus linear regret contradicts a uniform
\(O(k\log(ed/k))\) guarantee already at \(k=1\).  The unsigned and signed
power conventions yield the same second-stage prediction by the sign-matrix
identity in Section~\ref{sec:setup}, so the statement also covers literal
squaring at \(\alpha=2\).
\end{proof}

%% file: appendix/E_rank_bound.tex
\section{Proof of the rank upper bound}
\label{app:rank}

This appendix proves the structural upper bound in
Theorem~\ref{thm:rank-adaptive}.  The argument first characterizes exact
second-stage interpolation and then charges each positive-loss round to a new
row-space direction.

\subsection{Rank spending and interpolation}

\begin{lemma}[Realizable interpolation spends one rank per mistake]
\label{lem:rank-per-mistake}
Suppose $y_t=x_t^\top u\in[-1,1]$ for one fixed $u$, and a coefficient
$w_t$ satisfies $X_{<t}w_t=y_{<t}$ before round $t$.  Then a positive clipped
loss at round $t$ implies
$x_t\notin\operatorname{rowspan}(X_{<t})$.  Consequently,
\[
\sum_{t=1}^T\bigl(\clip(x_t^\top w_t)-y_t\bigr)^2
\le4\operatorname{rank}(X_{1:T}).
\]
\end{lemma}

\begin{proof}
If $x_t^\top=c^\top X_{<t}$ for some $c$, then interpolation and
realizability give
\[
x_t^\top w_t=c^\top X_{<t}w_t=c^\top y_{<t}
=c^\top X_{<t}u=x_t^\top u=y_t.
\]
  Thus the raw, and hence clipped, prediction is exact.  Let
  $r_t=\operatorname{rank}(X_{1:t})$ and set $r_0=0$.  A positive-loss round
  has $x_t\notin\operatorname{rowspan}(X_{<t})$, so adjoining that row gives
  $r_t-r_{t-1}=1$.  Therefore the number of positive-loss rounds is at most
  \[
  \sum_{t=1}^T(r_t-r_{t-1})=r_T
  =\operatorname{rank}(X_{1:T}).
  \]
  Both the clipped prediction and label lie in $[-1,1]$, so each of those
  rounds contributes squared loss at most four.
\end{proof}

\begin{lemma}[History interpolation by the powered rules]
\label{lem:powered-history-interpolation}
On any history $X,y$ with $y=Xe_1$, every powered-univariate and
powered-multivariate second-stage coefficient interpolates $y$.  The same is
true for a powered Pearson rule whenever its target prime is nonzero if
$y\ne0$.
\end{lemma}

\begin{proof}
  If $y=0$, then $A^\dagger y=0$ for every transformed design $A$, so the
  second-stage coefficient and fitted vector are both zero.  Suppose $y\ne0$.
  For univariate priming, the target historical
column is $y$, so its base and powered primes are both one; the transformed
design therefore contains $y$ as a column.  For multivariate priming, let
$p=X^\dagger y$.  Since $y\in\operatorname{col}(X)$,
\[
Xp=XX^\dagger y=y.
\]
Let $\pi_i=\operatorname{sign}(p_i)|p_i|^\alpha$ and define
\[
z_i=
\begin{cases}
p_i/\pi_i,&p_i\ne0,\\
0,&p_i=0.
\end{cases}
\]
Then
\[
\diag(\pi)z=p,
\qquad
X\diag(\pi)z=Xp=y.
\]
Finally, if the Pearson target prime $\pi_1$ is
nonzero, the first transformed column is $\pi_1y$, so the coefficient
$z=(1/\pi_1)e_1$ satisfies
\[
X\diag(\pi)z=y.
\]
In every case $y$ lies in the transformed column space, and
Lemma~\ref{lem:prime-refit-interpolation} shows that the Moore--Penrose refit
interpolates it exactly.
\end{proof}

\begin{proof}[Detailed proof of Theorem~\ref{thm:rank-adaptive}]
The generic inequality is Lemma~\ref{lem:rank-per-mistake}.  Lemma
\ref{lem:powered-history-interpolation} verifies its interpolation hypothesis
for the powered univariate and multivariate rules.  For Pearson priming, the
target-preserving convention in Theorem~\ref{thm:rank-adaptive} is the
nonzero-target-prime sufficient condition in that lemma.  More generally,
Lemma~\ref{lem:prime-refit-interpolation} gives the exact condition
$y_{<t}\in\operatorname{col}(X_{<t}D_t)$.  Finally,
\[
\operatorname{rank}(X_{1:T})\le\min\{T,d\}.
\]
\end{proof}

%% file: appendix/F_univariate_frontier.tex
\section{Proof of the powered-univariate exact frontier}
\label{app:univariate}

\subsection{A linear-dimensional triangular construction}

Let $m=\min\{d,T\}$.  We first give an active construction with dimension
and horizon $m$; padding with zero coordinates and then zero rows embeds it
into arbitrary $d$ and $T$.  Set
\[
x_{t,1}=y_t=1,
\qquad t=1,\ldots,m,
\]
and, for $k=2,\ldots,m$, define
\begin{equation}
x_{t,k}=\begin{cases}
1/(k-1),&t<k,\\
-1,&t=k,\\
0,&t>k.
\end{cases}
\label{eq:triangular-gadget}
\end{equation}
Every entry is rational and belongs to $[-1,1]$, and the comparator $e_1$
has zero loss.

\begin{lemma}[Exact rank of the triangular design]
\label{lem:triangular-rank}
The active $m\times m$ design in Equation~\eqref{eq:triangular-gadget} is
nonsingular.  More precisely,
\[
\det(X_{1:m})=(-1)^{m-1}m.
\]
Consequently its first $t$ rows are linearly independent for every
$t\le m$, so each active round introduces one new row-space direction.
\end{lemma}

\begin{proof}
For every $k=2,\ldots,m$, perform the column operation
\[
C_k\leftarrow C_k-\frac1{k-1}C_1.
\]
This operation does not change the determinant.  The
first row is then $(1,0,\ldots,0)$.  After deleting that row and the first
column, the remaining matrix is lower triangular: column $k$ is zero above
its trap row and has diagonal entry
$-1-(k-1)^{-1}=-k/(k-1)$.  Hence
\[
\det(X_{1:m})
=\prod_{k=2}^m\left(-\frac{k}{k-1}\right)
=(-1)^{m-1}m\ne0.
\]
The same elimination applied to the leading $t\times t$ submatrix gives
determinant $(-1)^{t-1}t$, proving independence of every active prefix.
\end{proof}

\begin{lemma}[Exact triangular prediction]
\label{lem:triangular-exact-prediction}
Fix $\alpha\ge1$, put $\beta=2\alpha-2$, and define
\[
S_t=1+\sum_{j=t-1}^{m-1}j^\beta.
\]
At every prediction round $2\le t\le m$, the powered-univariate primes and
second-stage coefficients satisfy
\begin{align}
p_{t,1}&=1,
&p_{t,k}&=0 &&(k<t),
&p_{t,k}&=k-1 &&(k\ge t),
\label{eq:triangular-primes}\\
w_{t,1}&=\frac1{S_t},
&w_{t,k}&=0 &&(k<t),
&w_{t,k}&=\frac{(k-1)^{2\alpha-1}}{S_t} &&(k\ge t).
\label{eq:triangular-weights}
\end{align}
Consequently,
\begin{align}
\widehat y_t
&=\frac{1-(t-1)^{2\alpha-1}
       +\sum_{j=t}^{m-1}j^\beta}{S_t},
\label{eq:triangular-prediction}\\
1-\widehat y_t
&=\frac{t(t-1)^\beta}{S_t},
\label{eq:triangular-powered-error}
\end{align}
and the clipped loss is
\begin{equation}
\ell_t^{\rm clip}
=\min\left\{\left(\frac{t(t-1)^\beta}{S_t}\right)^2,4\right\}.
\label{eq:triangular-clipped-loss}
\end{equation}
\end{lemma}

\begin{proof}
The target historical column equals the all-one label vector, so
$p_{t,1}=1$.  If $k<t$, nuisance coordinate $k$ has already passed its trap
round.  Its historical numerator and denominator are
\begin{align*}
X_{<t}(:,k)^\top y_{<t}
&=(k-1)\frac1{k-1}-1=0,\\
\|X_{<t}(:,k)\|_2^2
&=\frac{k-1}{(k-1)^2}+1
=\frac{k}{k-1}>0.
\end{align*}
Hence its univariate prime is zero.  If $k\ge t$,
all $t-1$ historical
entries in that column equal $1/(k-1)$, so
\[
p_{t,k}
=\frac{(t-1)/(k-1)}{(t-1)/(k-1)^2}=k-1.
\]
This proves Equation~\eqref{eq:triangular-primes}.

After powering, the nonzero transformed historical columns are multiples of
the all-one vector.  The target multiplier is one, while active nuisance
column $k$ has multiplier
\[
(k-1)^\alpha\frac1{k-1}=(k-1)^{\alpha-1}.
\]
Lemma
\ref{lem:collinear-min-norm} assigns transformed coefficient
$1/S_t$ to the target and
$(k-1)^{\alpha-1}/S_t$ to column $k$.  Multiplication by its powered prime
$(k-1)^\alpha$ gives Equation~\eqref{eq:triangular-weights}.

On the current row, coordinate $t$ equals $-1$, while every later nuisance
coordinate $k>t$ equals $1/(k-1)$.  Substitution of
Equation~\eqref{eq:triangular-weights} gives
Equation~\eqref{eq:triangular-prediction}.  Since
\[
S_t-\left(1-(t-1)^{2\alpha-1}
       +\sum_{j=t}^{m-1}j^\beta\right)
=(t-1)^\beta+(t-1)^{2\alpha-1}
=t(t-1)^\beta,
\]
Equation~\eqref{eq:triangular-powered-error} follows.  Its right-hand side is
nonnegative, so $\widehat y_t\le1$.  For a label equal to one, clipping a
raw prediction $q\le1$ yields loss $\min\{(1-q)^2,4\}$, which proves
Equation~\eqref{eq:triangular-clipped-loss}.
\end{proof}

\begin{lemma}[Linear loss and the unit-power count]
\label{lem:triangular-count}
For the active $m$-round construction and every fixed $\alpha\ge1$,
\[
\Reg_m^{\mathrm{clip},\UNI,\alpha}
\ge c_\alpha m,
\qquad
c_\alpha=\min\left\{\frac18,
\frac9{16}\left(\frac58\right)^{4\alpha-4}\right\}.
\]
At unit power the sharper bound
$\Reg_m^{\mathrm{clip},\UNI,1}\ge1+\lfloor m/2\rfloor$ holds.
\end{lemma}

\begin{proof}
First suppose $m\ge8$ and
$t\in\{\lceil3m/4\rceil,\ldots,m\}$.  Because $\beta\ge0$, every term in
$S_t$ is at most $(m-1)^\beta$, including the leading one.  Hence
\[
S_t\le(m-t+2)(m-1)^\beta
\le\left(\frac m4+2\right)(m-1)^\beta
\le\frac m2(m-1)^\beta.
\]
Moreover, $t\ge3m/4$ and, since $m\ge8$,
$t-1\ge3m/4-1\ge5m/8$.  Equation
\eqref{eq:triangular-powered-error} therefore gives
\[
1-\widehat y_t
\ge
\frac{(3m/4)(5m/8)^\beta}{(m/2)(m-1)^\beta}
\ge\frac32\left(\frac58\right)^\beta.
\]
The final lower bound is at most $3/2<2$, so
Equation~\eqref{eq:triangular-clipped-loss} implies loss at least
$\frac94(5/8)^{2\beta}$ on every such round.  Their number is
$m-\lceil3m/4\rceil+1\ge m/4$.  Summing them gives
\[
\Reg_m^{\mathrm{clip},\UNI,\alpha}
\ge\frac{9m}{16}\left(\frac58\right)^{2\beta}
=\frac{9m}{16}\left(\frac58\right)^{4\alpha-4}.
\]
If $m<8$, the empty-history prediction on the first round is zero and loses
one, which is at least $m/8$.  The two cases prove the definition of
$c_\alpha$.

When $\alpha=1$, $\beta=0$ and
Equation~\eqref{eq:triangular-powered-error} becomes
\[
1-\widehat y_t=\frac{t}{m-t+2}.
\]
This quantity is at least one exactly when
\[
t\ge\frac{m+2}{2}.
\]
There are $\lfloor m/2\rfloor$ integer rounds
$t\in\{2,\ldots,m\}$ satisfying this condition, and each loses at least one.
Adding the first-round loss proves
$1+\lfloor m/2\rfloor$; the same formula also covers $m=1$.
\end{proof}

\begin{proof}[Lower bound in Theorem~\ref{thm:univariate-exact}]
Take $m=\min\{d,T\}$ and run the active triangular construction in its first
$m$ coordinates and rounds.  If $d>m$, append zero feature coordinates;
Lemma~\ref{lem:zero-coordinate-padding} shows that every prediction in the
active prefix is unchanged.  If $T>m$, append $T-m$ zero inputs with zero
labels.  By Lemma~\ref{lem:zero-tail-extension}, these rounds preserve
realizability and contribute exactly zero loss, so
Lemma~\ref{lem:triangular-count} gives
\[
\Reg_T^{\mathrm{clip},\UNI,\alpha}
\ge c_\alpha m
=c_\alpha\min\{T,d\}.
\]
The unit-power bound is
embedded in the same way.  The matching upper order follows from
Theorem~\ref{thm:rank-adaptive}, since
$\operatorname{rank}(X_{1:T})\le\min\{T,d\}=m$.
\end{proof}

\begin{lemma}[Uniform scaling and clipping]
\label{lem:triangular-normalization}
Every row of the triangular construction has squared Euclidean norm at most
three.  After multiplying every feature row and label by $c=1/2$, the
powered-univariate coefficient is unchanged, every input has norm at most
one, and its clipped loss is at least one quarter of the corresponding
unscaled clipped loss.
\end{lemma}

\begin{proof}
On the first row,
\[
\|x_1\|_2^2
=1+\sum_{j=1}^{m-1}\frac1{j^2}
\le1+1+\int_1^\infty x^{-2}\,dx=3.
\]
For $t\ge2$, the target and trap coordinates contribute two, while the
remaining nonzero nuisance coordinates give
\[
\sum_{j=t}^{m-1}\frac1{j^2}
\le\int_{t-1}^\infty x^{-2}\,dx
=\frac1{t-1}\le1.
\]
Thus half-scaling makes every row norm at most $\sqrt3/2<1$ and leaves
$\|e_1\|_2=1$.

Let $X'=cX$ and $y'=cy$.  For each nonzero historical feature column,
\[
\frac{(cX_i)^\top(cy)}{\|cX_i\|_2^2}
=\frac{X_i^\top y}{\|X_i\|_2^2},
\]
  A zero historical feature column remains zero after scaling and receives
  prime zero both before and after scaling.  Thus every base, and hence
  powered, univariate prime is unchanged.  If $D$ is
the powered diagonal matrix, Lemma~\ref{lem:pseudoinverse-invariances} gives
\[
(cXD)^\dagger(cy)=c^{-1}(XD)^\dagger(cy)=(XD)^\dagger y.
\]
The second-stage coefficient is therefore unchanged, and each raw prediction
is multiplied by $c$.

It remains to compare clipping.  Lemma~\ref{lem:triangular-exact-prediction}
shows that every unscaled raw prediction $q$ satisfies $q\le1$.  For
$q\in[-1,1]$, neither $q$ nor $cq$ is clipped and
\[
(\clip(cq)-c)^2=c^2(q-1)^2
=c^2(\clip(q)-1)^2.
\]
If $-2\le q<-1$, then $cq\in[-1,-1/2)$ and the scaled loss is
$c^2(1-q)^2\ge4c^2=1$, equal to the right side
$c^2(-1-1)^2$.  If $q<-2$, the scaled prediction clips to $-1$ and its loss
is $(1+c)^2>1=4c^2$.  Hence in all cases
\[
(\clip(cq)-c)^2\ge c^2(\clip(q)-1)^2.
\]
With $c=1/2$, this is the claimed factor $1/4$.
\end{proof}

\begin{proof}[Euclidean-normalized part of
Theorem~\ref{thm:univariate-exact}]
  Let $m=\min\{T,d\}$ and half-scale every feature vector and label in the
  active $m$-round triangular construction.  Lemma
  \ref{lem:triangular-normalization} gives input norm at most one, keeps the
  powered-univariate coefficient unchanged, and retains at least one quarter
  of each active clipped loss.  If $d>m$, append $d-m$ identically zero feature
  coordinates; Lemma~\ref{lem:zero-coordinate-padding} preserves every active
  prediction.  If $T>m$, append $T-m$ rounds with zero input and zero label;
  Lemma~\ref{lem:zero-tail-extension} gives zero additional loss.  Throughout,
  the comparator remains $e_1$ and the scaled labels satisfy $y_t=x_{t,1}$.
  Lemma~\ref{lem:triangular-count} therefore yields
\[
\Reg_T^{\mathrm{clip},\UNI,\alpha}
\ge\frac{c_\alpha}{4}\min\{T,d\}.
\]
The rank upper bound in
  Theorem~\ref{thm:rank-adaptive} is unchanged, so the normalized worst-case
  order is $\Theta_\alpha(\min\{T,d\})$.
\end{proof}

%% file: appendix/G_pearson_exact.tex
\section{The exact frontier for Pearson priming}
\label{app:pearson-exact-frontier}

We prove Theorem~\ref{thm:pearson-exact-frontier} on the class
$x_t\in[-1,1]^d$, $y_t=x_{t,1}$.  Throughout this appendix the power is
one.  Target preservation has the meaning in Section~\ref{sec:setup}:
the target prime is nonzero whenever the historical label vector is
nonzero.  It need not equal one on a history with undefined correlation.
The lower-bound sequence will not depend on the finite past-only
totalization; target preservation is needed only for the matching upper
bound.

\subsection{A seeded affine--Hadamard construction}

Let $H\in\{\pm1\}^{L\times L}$ be a Sylvester Hadamard matrix of order
$L\ge64$, and write its rows as
$h_0=\boldsymbol 1,h_1,\ldots,h_{L-1}$.  We use that these rows are
characters: they are orthogonal, every nonconstant row is balanced, and
their entrywise products are again rows of $H$.

Set
\[
c_0=\frac12,
\qquad A=\frac23,
\qquad K=128,
\qquad B_0=KL,
\qquad R=\frac{L}{64}.
\]
The active dimension is $1+2L$.  Index the nuisance coordinates by pairs
$(j,+),(j,-)$ for $j\in[L]$.  A block, specified by
$q_b\in\{\pm1\}$ and $s_b\in\{\pm1\}^L$, contains two examples with labels
$+1$ and $-1$.  Their target coordinate equals the label, and
\begin{align}
x^{+}_{b,(j,+)}&=A(q_b+c_0s_{b,j}),
&x^{+}_{b,(j,-)}&=A(-q_b+c_0s_{b,j}),
\label{eq:pearson-block-plus}\\
x^{-}_{b,(j,+)}&=A(q_b-c_0s_{b,j}),
&x^{-}_{b,(j,-)}&=A(-q_b-c_0s_{b,j}).
\label{eq:pearson-block-minus}
\end{align}
Thus every input entry belongs to
$\{-1,-1/3,1/3,1\}$.  Take $q_b=(-1)^b$.  The first $B_0$ blocks use
$s_b=h_0$; the next $R$ blocks use $h_1,\ldots,h_R$.  We charge only the
first, positive-label example in each of the latter blocks.  The whole
sequence is fixed before play, and $e_1$ realizes every label.

\begin{lemma}[Exact primes and compression]
\label{lem:pearson-seeded-compression}
Suppose the history consists of $B=B_0+r$ completed blocks, where
$0\le r<R$, and put
\[
z_{r,j}=\sum_{i=1}^r h_{i,j}.
\]
If $V_B=\sum_{b<B}(q_b-\bar q_B)^2$, then the two nuisance coordinates in
pair $j$ have the same Pearson prime
\begin{equation}
p_{r,j}
=\frac{c_0(B_0+z_{r,j})}
{\sqrt{B(V_B+Bc_0^2)}}.
\label{eq:pearson-seeded-prime}
\end{equation}
All target and nuisance empirical variances are positive.  Moreover, the
second minimum-norm fit and the next charged prediction are exactly those of
minimum-norm interpolation of the all-one response by the compressed rows
\begin{equation}
v_r(s)=
\bigl(1,\gamma_{r,1}s_1,\ldots,\gamma_{r,L}s_L\bigr),
\qquad
\gamma_{r,j}=\sqrt2 A c_0p_{r,j}.
\label{eq:pearson-compressed-row}
\end{equation}
The distinct historical signs are $h_0,\ldots,h_r$, and the next sign is
$h_{r+1}$.
\end{lemma}

\begin{proof}
Every completed block has one label of each sign, so the target has mean zero
and squared norm $2B$.  The alternating $q_b$ sequence satisfies
\[
V_B=\begin{cases}
B,&B\text{ even},\\
B-B^{-1},&B\text{ odd}.
\end{cases}
\]
For coordinate $(j,+)$, summing feature--label products within a block gives
$2Ac_0s_{b,j}$, while its centered squared norm over the history is
$2A^2(V_B+Bc_0^2)$.  Coordinate $(j,-)$ has the same two quantities.  This
proves Equation~\eqref{eq:pearson-seeded-prime}.  Since
$B_0+z_{r,j}\ge B_0-r>0$ and the displayed centered norm is positive, every
prime used on a charged round is defined.

For the second assertion, apply the orthogonal change of coefficient
coordinates that takes the sum and difference inside each nuisance pair.
The difference column is constant within each two-example block; the sum
column changes sign with the label.  These block-mean and block-difference
subspaces are orthogonal, and the label belongs to the latter.  The
minimum-norm coefficient on every block-mean column is therefore zero.
More explicitly, if the original transformed coefficients on pair $j$ are
$z_{j,+},z_{j,-}$, put
$u_j=(z_{j,+}+z_{j,-})/\sqrt2$ and
$v_j=(z_{j,+}-z_{j,-})/\sqrt2$.  The two interpolation constraints in
block $b$ become
\[
 z_1+\sum_j\sqrt2 A c_0p_{r,j}s_{b,j}u_j=1,
 \qquad
 \sum_j\sqrt2 A p_{r,j}q_bv_j=0.
\]
The squared coefficient norm is $z_1^2+\|u\|_2^2+\|v\|_2^2$.
The second constraint is homogeneous, so its minimum-norm solution is
$v=0$.  The remaining constraints are exactly
Equation~\eqref{eq:pearson-compressed-row}.  They are feasible: the
positive weights $\gamma_{r,j}$ preserve the row rank of the distinct
Hadamard rows.  Repeated seed blocks therefore repeat the same exact
constraint and do not change the minimum-norm solution.  The next positive
example has no block-mean contribution because $v=0$, so the compression
also preserves its prediction, not just the historical fit.
\end{proof}

For our choices of $A$ and $c_0$, define
\begin{equation}
\lambda_r=\frac{1}{18B(V_B+B/4)}.
\label{eq:pearson-lambda}
\end{equation}
Then Lemma~\ref{lem:pearson-seeded-compression} gives
\begin{equation}
\gamma_{r,j}^2=\lambda_r(B_0+z_{r,j})^2.
\label{eq:pearson-gamma-square}
\end{equation}

\begin{lemma}[Weighted Hadamard certificate]
\label{lem:pearson-weighted-hadamard}
Let $n=r+1$ be the number of distinct compressed historical rows.  Their
nuisance Gram matrix has a common diagonal
\[
\Lambda_r=\lambda_rL(B_0^2+r)\ge\frac{L}{24},
\]
and every off-diagonal entry has magnitude less than $1/1024$.  The same
upper bound holds for every nuisance inner product between the next query
row and an old row.  Consequently the raw prediction on each of the first
$R=L/64$ query blocks has absolute value less than $1/2$.
\end{lemma}

\begin{proof}
Let $S_r$ have rows $h_0,\ldots,h_r$.  By
Equation~\eqref{eq:pearson-gamma-square}, the nuisance Gram matrix is
$S_r\operatorname{diag}(\gamma_r^2)S_r^\top$.  Orthogonality and balance give
\[
\sum_jz_{r,j}=0,
\qquad
\sum_jz_{r,j}^2=rL,
\]
which proves the diagonal formula.

For distinct $h_u,h_v$, expand $(B_0+z_{r,j})^2$.  The constant term
vanishes.  In the linear term, character closure permits at most one index
$i$ with $h_i=h_u\odot h_v$.  In the quadratic term, each $i$ permits at
most one $k$ with $h_i\odot h_k=h_u\odot h_v$.  Therefore every old--old or
new--old nuisance cross inner product is bounded by
\begin{equation}
\varepsilon_r
=\lambda_rL(2B_0+r).
\label{eq:pearson-cross-bound}
\end{equation}

Since $B_0\le B\le B_0+L$ and $3B/4\le V_B\le B$,
\begin{align}
\Lambda_r
&\ge\frac{2}{45}L\left(\frac{K}{K+1}\right)^2
\ge\frac{L}{24},
\label{eq:pearson-diagonal-lower}\\
\varepsilon_r
&\le\frac{2K+1}{18K^2}
<\frac1{1024}.
\label{eq:pearson-offdiag-upper}
\end{align}
Write the nuisance Gram matrix as $\Lambda_rI+E_r$.  Since
$n\le L/64$, Gershgorin's theorem gives
\[
\lambda_{\min}(\Lambda_rI+E_r)>\frac{L}{25}.
\]
Adding the positive-semidefinite target contribution
$\boldsymbol1\boldsymbol1^\top$ preserves this bound.  If $G_r$ denotes the
full compressed Gram matrix and $k_r$ the cross-Gram vector of the new row,
then
\[
\|G_r^{-1}\|_{\mathrm{op}}<\frac{25}{L},
\qquad
\|k_r\|_2
\le\left(1+\frac1{1024}\right)\sqrt n
<\frac{26}{25}\sqrt n.
\]
The minimum-norm prediction is $k_r^\top G_r^{-1}\boldsymbol1$, and hence
\[
\left|k_r^\top G_r^{-1}\boldsymbol1\right|
<\frac{26n}{L}
\le\frac{26}{64}
<\frac12.
\]
\end{proof}

\begin{proof}[Proof of Theorem~\ref{thm:pearson-exact-frontier}]
Lemma~\ref{lem:pearson-weighted-hadamard} shows that each charged label is
$+1$ while its raw prediction lies in $[-1/2,1/2]$.  Clipping is inactive,
so every charged round loses at least $1/4$.  Thus the active construction
has regret at least
\[
\frac{R}{4}=\frac{L}{256}.
\]
It uses dimension $1+2L$ and fewer than
\[
2(B_0+R)<257L
\]
rounds.

For arbitrary $T,d$, put
\[
Q=\min\left\{\frac{d-1}{2},\frac{T}{257}\right\}.
\]
If $Q\ge64$, take the largest power of two $L\le Q$.  Then
$L\ge\min\{T,d\}/514$, and zero-coordinate and zero-tail padding give
\[
\operatorname{Reg}^{\mathrm{clip}}_T
\ge\frac{1}{131584}\min\{T,d\}.
\]
To justify the size comparison, write $m=\min\{T,d\}$.  In this regime
$d\ge129$, hence $(d-1)/2\ge d/257$ and $Q\ge m/257$.
The largest power of two below $Q$ is at least $Q/2$, as required.
Extra zero coordinates make zero columns even if their undefined primes
are assigned arbitrary finite values; the minimum-norm fit gives these
columns zero coefficient.  Appended zero examples have zero prediction
and label.  Thus both forms of padding preserve the active losses.

If $Q<64$, then $d<129$ or $T<16448$, so $m<16448$.
The empty-history prediction on $(e_1,1)$ is zero for every finite
totalization: the pseudoinverse of the empty design maps the empty
response to zero.  Its loss is one, which exceeds $m/131584$; follow it
by zero examples and coordinates as needed.  In the Hadamard regime,
uniformity over totalizations follows instead from the positive variances
in Lemma~\ref{lem:pearson-seeded-compression}.  No positive-variance claim
is needed for the empty-history fallback.

For the upper bound, if $y_{<t}\ne0$, target preservation places a nonzero
multiple of $y_{<t}$ among the columns of $X_{<t}D_t$.  The second fit
therefore interpolates.  If $y_{<t}=0$, its minimum-norm fit is zero and
also interpolates.  Theorem~\ref{thm:rank-adaptive} now gives
$4\operatorname{rank}(X_{1:T})\le4\min\{T,d\}$.  The upper bound is taken
over the same realizable class as the lower bound, not over arbitrary
label sequences.

Finally, target preservation cannot be dropped.  Consider the legal
totalization that assigns zero to an undefined target prime.  In dimension
one, the sequence $x_t=y_t=1$ keeps that prime undefined and equal to zero,
so the learner predicts zero and loses one on every round.
\end{proof}

%% file: appendix/H_multivariate_euclidean.tex
\section{A Euclidean-unit lower bound for multivariate priming}
\label{app:multivariate-euclidean}

We prove Theorem~\ref{thm:multivariate-euclidean} at unit power.  Two unit
rows first realize a prescribed minimum-norm state.  An orthonormal tail
then leaves the first-stage prime fixed but repeatedly incurs nonzero
second-stage prediction error.  All labels equal the first coordinate.

\subsection{Realizing a prescribed state}

For a subspace $S$, let $P_S$ be its orthogonal projector.  For a vector
$c$, write $D_c=\diag(c)$; division of vectors below is entrywise.

\begin{lemma}[Two-row state realization]
\label{lem:euclidean-state-realization}
Suppose $v\in(0,\infty)^d$, $q\in\R^d$, and $s>0$ satisfy
\begin{equation}
 \sum_i v_i=1,\quad v_1=s,\quad
 q_1=\|q\|_2^2,\quad v^\top q=s,\quad
 q_1=s\sum_iq_i.
 \label{eq:euclidean-state-moments}
\end{equation}
Set $c_i=\sqrt{v_i}$ and
$S=\operatorname{span}\{c,D_c^{-1}q\}$.
Any full-row-rank history spanning $S$, with labels given by its first
column, has multivariate prime and unit-power refit
\begin{equation}
 p=\sqrt{s}\,c,
 \qquad w=\frac1{\sqrt{s}}D_cq.
 \label{eq:euclidean-state-fit}
\end{equation}
If $v,q$ are algebraic, $S$ admits an algebraic orthonormal basis of at most
two rows.
\end{lemma}

\begin{proof}
The vector $\sqrt{s}c$ belongs to $S$.  Its residual from $e_1$ is
orthogonal to both generators, since
\[
 c^\top(e_1-\sqrt{s}c)=0,
 \qquad
 (D_c^{-1}q)^\top(e_1-\sqrt{s}c)
 =\frac{q_1}{\sqrt{s}}-\sqrt{s}\sum_iq_i=0.
\]
Consequently $X^\dagger Xe_1=P_Se_1=\sqrt{s}c=p$.
The transformed row space is
$D_pS=\operatorname{span}\{v,q\}$ (identifying row and column vectors).
Moreover $q$ is the projection of $e_1$ onto this space: it belongs to the
space, and
$v^\top(e_1-q)=s-v^\top q=0$ and
$q^\top(e_1-q)=q_1-\|q\|_2^2=0$.
Since $p_1=s$, the transformed coefficient $e_1/s$ interpolates the labels.
The minimum-norm transformed coefficient is its projection onto $D_pS$,
namely $q/s$.  Multiplication by $D_p$ gives
Equation~\eqref{eq:euclidean-state-fit}.
Finally, Gram--Schmidt uses only arithmetic and square roots on independent
algebraic generators.  Its nonzero normalized rows are algebraic unit
vectors and span the same space.
\end{proof}

\subsection{An explicit alternating-pair state}

Fix an even integer $K\ge4$ and set
\begin{equation}
 s=\frac1{64K^2},\qquad N=8192K^4=\frac2{s^2},\qquad b_0=\frac1{16}.
 \label{eq:euclidean-parameters}
\end{equation}
For $g=0,\ldots,K-1$, put $a_g=1$ if $g$ is even and $a_g=2$ otherwise,
and define
\[
 Q_g=\sqrt{\frac{a_g}{256K}}>0,\qquad V_g=\frac1{64Ka_g}.
\]
Pair $g$ consists of two coordinates with $q$-values $Q_g,-Q_g$ and
$v$-values $V_g/2,V_g/2$.  Exactly half the $a_g$ have each value, so
\begin{equation}
 2\sum_gQ_g^2=\frac3{256},\qquad
 \sum_gV_g=\frac3{256},\qquad
 \frac{V_gQ_g^2}{s}=b_0^2.
 \label{eq:euclidean-pair-moments}
\end{equation}
The remaining coordinates enforce the moments needed for realization.
Define
\begin{align}
 L&=\frac1{2s}-\frac34, & R&=\frac{45}{256},\notag\\
 \Delta&=\sqrt{N((N+1)R-L^2)}, &
 B&=\frac{NL-\Delta}{N(N+1)}, & D&=L-NB.
 \label{eq:euclidean-bulk}
\end{align}
There is one target coordinate with $(q_1,v_1)=(1/2,s)$, one carrier with
$(q,v)=(1/4,V_{\rm car})$, $N$ bulk coordinates each with
$(q,v)=(B,1/(4N))$, one extra coordinate with $(q,v)=(D,s)$, and a sink
with $(q,v)=(0,V_{\rm sink})$, where
\begin{equation}
 V_{\rm car}=2s-B-4sD,\qquad
 V_{\rm sink}=1-2s-\frac3{256}-\frac14-V_{\rm car}.
 \label{eq:euclidean-masses}
\end{equation}
Thus the dimension is
\begin{equation}
 d_K=8192K^4+2K+4.
 \label{eq:euclidean-dimension}
\end{equation}

All masses are strictly positive.  Indeed, $s\le1/1024$ and direct expansion gives
\[
 (N+1)R-L^2=\frac{13}{128s^2}+\frac3{4s}-\frac{99}{256}>0,
 \qquad L^2>R.
\]
The second inequality implies $NL>\Delta$, so
$0<B<L/N<s/4$.  Also
\[
 D=\frac{L+\Delta}{N+1}>0,
 \qquad
 D<\frac{s}{4}+\sqrt{45/256}<\frac{21}{50}.
\]
It follows that
$V_{\rm car}>(2-1/4-42/25)s>0$ and $V_{\rm car}<2s$.
Hence $V_{\rm sink}>1-4s-3/256-1/4>0$.  All individual $v$-values
are strictly positive.

Substitution in Equation~\eqref{eq:euclidean-bulk} yields
$NB+D=L$ and $NB^2+D^2=R$.  The paired coordinates cancel in the first
moment of $q$, and their second moment is $3/256$.  Consequently
\[
 \sum_iq_i=\frac34+L=\frac1{2s},\qquad
 \|q\|_2^2=\frac14+\frac1{16}+\frac3{256}+\frac{45}{256}=\frac12.
\]
The definition of the sink gives $\sum_iv_i=1$, while
\[
 v^\top q=\frac{s}{2}+\frac{V_{\rm car}}4+\frac B4+sD=s.
\]
All conditions of Lemma~\ref{lem:euclidean-state-realization} hold.
The vectors $v,q$ are independent because $q$ changes sign within a pair
whereas $v$ does not.  Choose the two orthonormal build rows by
Gram--Schmidt on $c,D_c^{-1}q$ in that order, and label each by its first
coordinate.  This specifies a fixed algebraic two-example history.

Let $f_g=(e_{g,+}-e_{g,-})/\sqrt2$.  The prime is symmetric on each pair,
and Equation~\eqref{eq:euclidean-state-fit} gives
\begin{equation}
 f_g^\top p=0,\qquad f_g^\top w=b_0
 \quad(g=0,\ldots,K-1).
 \label{eq:euclidean-stored-prediction}
\end{equation}
The next step releases these nuisance predictions on new unit directions.

\subsection{The orthonormal tail and its exact predictions}

For $j=0,\ldots,K-2$, set $A_j=\sqrt{\sum_{i=0}^ja_i^2}>0$ and define
\begin{equation}
 u_j=\frac{a_{j+1}}{A_j\sqrt{A_j^2+a_{j+1}^2}}
       \sum_{i=0}^ja_if_i
       -\frac{A_j}{\sqrt{A_j^2+a_{j+1}^2}}f_{j+1}.
 \label{eq:euclidean-helmert}
\end{equation}
Each row has norm one and zero dot product with $\sum_ga_gf_g$.
If $i<j$, the coefficients of $u_j$ on the support of $u_i$ are
proportional to $(a_0,\ldots,a_{i+1})$, so $u_i^\top u_j=0$.
The rows are therefore orthonormal.  They are orthogonal to $c$ by pair
symmetry; they are orthogonal to $D_c^{-1}q$ because
\[
 f_g^\top D_c^{-1}q=\frac{2Q_g}{\sqrt{V_g}}=a_g.
\]
Thus they are orthogonal to the build space and have zero target
coordinate.  Give each tail row label zero.  Since the projection of $e_1$
onto every new tail direction is zero, the ordinary prime stays equal to
$p$ throughout the tail.

The following calculation includes the Schur-complement step needed to
control the second fit after every prefix.  Represent tail vectors by
their coefficients in the $f_g$ basis, and put
\[
 M=\diag(a_0^{-1},\ldots,a_{K-1}^{-1}),\qquad
 \kappa=b_0s^{3/2},\qquad D_p=\diag(p).
\]
The definitions give
\begin{equation}
 f_g^\top D_p^2f_h=\frac{\kappa}{a_g}\mathbf1\{g=h\},\qquad
 (D_c^{-1}q)^\top D_p^2f_g=\kappa,\qquad
 c^\top D_p^2f_g=0.
 \label{eq:euclidean-gram-blocks}
\end{equation}

\begin{lemma}[Tail prediction by a Schur complement]
\label{lem:euclidean-schur}
Immediately before $u_j$, let $U$ have rows $u_0,\ldots,u_{j-1}$ in pair
coordinates, and let $\bar u_j$ be the pair-coordinate vector of $u_j$.
There are scalars $F_j\ge b_0$ such that
\begin{equation}
 \widehat y_j=F_j\mathcal B_j,\qquad
 \mathcal B_j=\bar u_j^\top
 [\one-MU^\top(UMU^\top)^{-1}U\one].
 \label{eq:euclidean-schur-prediction}
\end{equation}
For $j=0$, the term involving $U$ is zero.  For every $j$,
\begin{equation}
 \mathcal B_j=
 \sqrt{\frac{S_2}{S_2+a_{j+1}^2}}
 \left(a_{j+1}\frac{S_2}{S_3}-1\right),\quad
 S_2=\sum_{i=0}^ja_i^2,\quad S_3=\sum_{i=0}^ja_i^3.
 \label{eq:euclidean-tail-factor}
\end{equation}
\end{lemma}

\begin{proof}
Both fits depend on the exact interpolation constraints, so replacing the
two build rows by their generators $c,D_c^{-1}q$, with the corresponding
first-coordinate labels, changes neither fit.  All primes are positive,
and all historical rows are independent.  Their transformed Gram matrix
is therefore positive definite at every tail prefix.

Eliminate the $c$ row from this Gram system.  Let $\mathsf A>0$ be the resulting
scalar entry for $D_c^{-1}q$ and $Y$ its resulting right-hand side.  These
two quantities do not depend on the tail prefix.  By
Equation~\eqref{eq:euclidean-gram-blocks}, the remaining system is
\[
 \begin{pmatrix}\mathsf A&\kappa h^\top\\\kappa h&H\end{pmatrix}
 \begin{pmatrix}\eta\\\xi\end{pmatrix}
 =\begin{pmatrix}Y\\0\end{pmatrix},\qquad
 h=U\one,\quad H=\kappa UMU^\top.
\]
The second equation gives $\xi=-\kappa H^{-1}h\eta$, and the first gives
\[
 \eta=\frac{Y}{\mathsf A-\kappa^2h^\top H^{-1}h}.
\]
The new row has cross-Gram entries $0$ with $c$,
$\kappa\bar u_j^\top\one$ with $D_c^{-1}q$, and
$\kappa\bar u_j^\top MU^\top$ with the old tail.
Multiplying these entries by the displayed solution proves
Equation~\eqref{eq:euclidean-schur-prediction}, with
\[
 F_j=\frac{\kappa Y}{\mathsf A-\kappa^2h^\top H^{-1}h}.
\]
Its denominator is positive by positive definiteness.  With no tail,
Equation~\eqref{eq:euclidean-stored-prediction} gives
$\kappa Y/\mathsf A=b_0>0$.  Also $h^\top H^{-1}h\ge0$, so $F_j\ge b_0$.
For completeness, this quadratic form cannot decrease as rows are added:
it is the maximum of $2h^\top z-z^\top Hz$ over the old dual variables,
and embedding those variables with a final zero is feasible in the
enlarged maximum.  Thus the Schur denominator decreases while remaining
strictly positive.

To compute the bracket, restrict to the first $j+1$ coordinates.  The old
Helmert rows span $a^\perp$, where $a=(a_0,\ldots,a_j)^\top$.
For a positive diagonal matrix $M_0$ on this prefix,
\[
 U^\top(UM_0U^\top)^{-1}U
 =M_0^{-1}-\frac{M_0^{-1}aa^\top M_0^{-1}}
                         {a^\top M_0^{-1}a}.
\]
Indeed multiplication on both sides by $M_0^{1/2}$ produces the
orthogonal projector onto $(M_0^{-1/2}a)^\perp$.
Taking $M_0=\diag(a_i^{-1})$, the prefix of the bracket in
Equation~\eqref{eq:euclidean-schur-prediction} is $aS_2/S_3$.
Its next coordinate is one.  Substitution of
Equation~\eqref{eq:euclidean-helmert} gives
Equation~\eqref{eq:euclidean-tail-factor}.  The same formula is valid at
$j=0$, where the prefix bracket has its sole coordinate equal to one.
\end{proof}

\subsection{Clipped loss and arbitrary sizes}

\begin{proof}[Proof of Theorem~\ref{thm:multivariate-euclidean}]
For an odd index $j=2r-1\le K-2$, the next scale is $a_{j+1}=1$ and
the prefix contains $r$ copies of each scale.  Hence $S_2=5r$, $S_3=9r$,
and
\[
 |\mathcal B_j|=\sqrt{\frac{5r}{5r+1}}\frac49
 \ge\sqrt{\frac56}\frac49.
\]
Lemma~\ref{lem:euclidean-schur} implies
\begin{equation}
 \widehat y_j^2\ge b_0^2\frac56\frac{16}{81}
 =\frac5{7776}=:c_0.
 \label{eq:euclidean-charged-loss}
\end{equation}
The label is zero and $c_0<1$, so clipping gives loss
$\min\{\widehat y_j^2,1\}\ge c_0$ even when the raw prediction is
outside $[-1,1]$.  There are $(K-2)/2$ such indices.  The two build
losses and the other tail losses are nonnegative, giving
\[
 \Reg_{K+1}^{\mathrm{clip}}(e_1)\ge\frac{c_0}{2}(K-2).
\]
Every active row is a fixed algebraic unit vector, and $e_1$ realizes all
labels exactly.  No adaptive choice of row or sign has been used.

Here is an explicit passage to arbitrary $T,d\ge1$.  Let
$m=\min\{T,d^{1/4}\}$ and
$Q=\min\{T-1,(d/8193)^{1/4}\}$.  If $Q\ge4$, choose the largest even
integer $K\le Q$.  Then $K\ge Q/2$, $K+1\le T$, and
$d_K\le8193K^4\le d$.  Moreover $T\ge5$, so
$Q\ge m/8193^{1/4}$.  Since $K-2\ge K/2$, the lower bound is at least
\[
 \frac{c_0}{8\,8193^{1/4}}m.
\]
Pad by zero coordinates and then zero examples.  Zero columns have
minimum-norm coefficient zero in both fits, and a zero example has zero
prediction, so padding does not change any active loss or norm constraint.
If $Q<4$, then either $T<5$ or $d^{1/4}<4\,8193^{1/4}<40$.
Thus $m<40$, and the single first example $(e_1,1)$, followed by zeros,
has loss one and gives the same lower bound after choosing
$c_{\rm M}=\min\{1/40,c_0/(8\,8193^{1/4})\}$.
This proves the theorem in all dimensions and horizons.
\end{proof}

%% file: appendix/G_ridge_robustness.tex
\section{Ridge robustness and the regularization crossover}
\label{app:ridge}

This appendix proves Theorem~\ref{thm:ridge-univariate} and records the
corresponding Hadamard consequence for all three prime rules.  Throughout,
ridge is applied to the transformed second-stage coefficient.  Given a
powered prime matrix $D$, historical design $X$, labels $y$, and
$A=XD$, define for $\lambda>0$
\begin{equation}
z_\lambda
=\arg\min_z\{\lVert Az-y\rVert_2^2+\lambda\lVert z\rVert_2^2\}
=A^\top(AA^\top+\lambda I)^{-1}y,
\qquad
w_\lambda=Dz_\lambda.
\label{eq:appendix-ridge-protocol}
\end{equation}
At $\lambda=0$, set $z_0=A^\dagger y$.  All ridge parameters below are
finite and nonnegative.

\subsection{A resolvent form of target mass}

Suppose the first historical feature column equals $y$.  Separate the
transformed target column from the nuisance columns:
\[
A=[\pi_1y,B].
\]

\begin{lemma}[Ridge target-mass identity]
\label{lem:ridge-target-mass}
For $\lambda>0$, let
\[
C_\lambda=BB^\top+\lambda I,
\qquad
q_\lambda=y^\top C_\lambda^{-1}y.
\]
Then
\begin{equation}
z_{\lambda,1}
=\frac{\pi_1q_\lambda}{1+\pi_1^2q_\lambda},
\qquad
w_{\lambda,1}
=\frac{\pi_1^2q_\lambda}{1+\pi_1^2q_\lambda}.
\label{eq:appendix-ridge-target-mass}
\end{equation}
In particular, $0\le w_{\lambda,1}<1$.
\end{lemma}

\begin{proof}
The matrix $C_\lambda$ is positive definite, and
\[
AA^\top=C_\lambda+\pi_1^2yy^\top.
\]
Sherman--Morrison gives
\[
(C_\lambda+\pi_1^2yy^\top)^{-1}y
=\frac{C_\lambda^{-1}y}
{1+\pi_1^2y^\top C_\lambda^{-1}y}.
\]
The first row of $A^\top$ is $\pi_1y^\top$.  Taking the first coordinate
in Equation~\eqref{eq:appendix-ridge-protocol} and then multiplying by
$\pi_1$ proves Equation~\eqref{eq:appendix-ridge-target-mass}.  Positive
definiteness of $C_\lambda^{-1}$ gives $q_\lambda\ge0$, so the displayed
fraction belongs to $[0,1)$.
\end{proof}

\begin{lemma}[Nuisance certificate and monotonicity]
\label{lem:ridge-certificate}
If $Bu=y$, then, for every $\lambda>0$,
\begin{equation}
q_\lambda\le\lVert u\rVert_2^2.
\label{eq:ridge-certificate-cost}
\end{equation}
Moreover,
\begin{equation}
\frac{d}{d\lambda}q_\lambda
=-y^\top(BB^\top+\lambda I)^{-2}y\le0.
\label{eq:ridge-q-monotone}
\end{equation}
If $y\in\operatorname{col}(B)$, then
\begin{equation}
\lim_{\lambda\downarrow0}q_\lambda
=y^\top(BB^\top)^\dagger y
=\min\{\lVert u\rVert_2^2:Bu=y\}.
\label{eq:ridge-q-endpoint}
\end{equation}
Thus Lemma~\ref{lem:ridge-target-mass} extends continuously to the
Moore--Penrose endpoint.
\end{lemma}

\begin{proof}
Set $v=C_\lambda^{-1}y$.  Since $y=Bu$,
\[
q_\lambda=y^\top v=u^\top B^\top v.
\]
The same quantity also satisfies
\begin{equation}
q_\lambda
=v^\top C_\lambda v
=\lVert B^\top v\rVert_2^2+\lambda\lVert v\rVert_2^2
\ge\lVert B^\top v\rVert_2^2.
\label{eq:ridge-q-energy}
\end{equation}
If $q_\lambda=0$, Equation~\eqref{eq:ridge-certificate-cost} is immediate.
Otherwise, Cauchy--Schwarz and Equation~\eqref{eq:ridge-q-energy} give
\[
q_\lambda
\le\lVert u\rVert_2\lVert B^\top v\rVert_2
\le\lVert u\rVert_2\sqrt{q_\lambda},
\]
and division by $\sqrt{q_\lambda}$ proves the certificate bound.
Differentiating the matrix inverse proves
Equation~\eqref{eq:ridge-q-monotone}.

For the endpoint, take a singular-value decomposition of $B$.  Because
$y\in\operatorname{col}(B)$, it has no component in $\ker(B^\top)$.
Each nonzero singular component of $q_\lambda$ therefore converges from
$(\sigma_i^2+\lambda)^{-1}$ to $\sigma_i^{-2}$.  This proves the first
equality in Equation~\eqref{eq:ridge-q-endpoint}; the second is the
minimum-norm formula for the consistent system $Bu=y$.
\end{proof}

In particular, if $Bu=y$ and
$\pi_1^2\lVert u\rVert_2^2\le\theta$, then for every $\lambda\ge0$,
\begin{equation}
0\le w_{\lambda,1}\le\frac{\theta}{1+\theta}.
\label{eq:ridge-target-cap}
\end{equation}

\subsection{The triangular sequence for every ridge schedule}

Let $m=\min\{T,d\}$ and use the triangular sequence in
Equation~\eqref{eq:triangular-gadget}.  Fix $\alpha\ge1$, put
$\beta=2\alpha-2$, and, before active round $t\ge2$, write
\[
s=t-1,
\qquad
S_t=1+\sum_{j=t-1}^{m-1}j^\beta.
\]

\begin{lemma}[Exact triangular ridge shrinkage]
\label{lem:triangular-ridge-shrinkage}
The powered transformed historical design has the rank-one form
\begin{equation}
A_t=\one_s r_t^\top,
\qquad
\lVert r_t\rVert_2^2=S_t,
\label{eq:ridge-triangular-rank-one}
\end{equation}
where
\[
r_t=(1,0,\ldots,0,(t-1)^{\alpha-1},\ldots,(m-1)^{\alpha-1})^\top.
\]
For every $\lambda>0$,
\begin{equation}
z_t^\lambda=\frac{s}{\lambda+sS_t}r_t.
\label{eq:ridge-triangular-z}
\end{equation}
If $w_t^0$ and $\widehat y_t^0$ denote the Moore--Penrose coefficient and
current raw prediction, then for every $\lambda\ge0$,
\begin{equation}
w_t^\lambda=\rho_t(\lambda)w_t^0,
\qquad
\widehat y_t^\lambda=\rho_t(\lambda)\widehat y_t^0,
\qquad
\rho_t(\lambda)=\frac{sS_t}{\lambda+sS_t}.
\label{eq:ridge-triangular-scaling}
\end{equation}
The target resolvent and target weight are
\begin{equation}
q_t(\lambda)=\frac{s}{\lambda+s(S_t-1)},
\qquad
w_{t,1}^\lambda=\frac{s}{\lambda+sS_t}.
\label{eq:ridge-triangular-target}
\end{equation}
\end{lemma}

\begin{proof}
Lemma~\ref{lem:triangular-exact-prediction} gives the historical primes:
the target base prime is one, expired coordinates have prime zero, and future
coordinate $k$, indexed by $j=k-1\ge t-1$, has base prime $j$.
Its powered transformed historical column is therefore
$j^\alpha(1/j)\one_s=j^{\alpha-1}\one_s$.  This proves
Equation~\eqref{eq:ridge-triangular-rank-one}, including its squared norm.

Since
\[
A_tA_t^\top=S_t\one_s\one_s^\top,
\]
the vector $\one_s$ is an eigenvector with eigenvalue $sS_t$.  Hence
\[
(A_tA_t^\top+\lambda I)^{-1}\one_s
=\frac{\one_s}{\lambda+sS_t}.
\]
Multiplication by $A_t^\top=r_t\one_s^\top$ proves
Equation~\eqref{eq:ridge-triangular-z}.  At $\lambda=0$, the minimum-norm
coefficient is $z_t^0=r_t/S_t$, so their ratio is
$\rho_t(\lambda)$.  Multiplication by the same prime diagonal and evaluation
on the same current row preserve this scalar, proving
Equation~\eqref{eq:ridge-triangular-scaling}.

For Equation~\eqref{eq:ridge-triangular-target}, remove the target entry of
$r_t$.  The nuisance transformed design $B_t$ then satisfies
$B_tB_t^\top=(S_t-1)\one_s\one_s^\top$.  Applying its resolvent to
$y=\one_s$ gives $q_t(\lambda)=s/[\lambda+s(S_t-1)]$.  Substitution into
Lemma~\ref{lem:ridge-target-mass}, with $\pi_1=1$, yields the target weight.
\end{proof}

The crossover in Equation~\eqref{eq:ridge-phase-scale} is now immediate:
\[
\rho_t(\lambda)
=\frac{1}{1+\lambda/\lambda_c(t)},
\qquad
\lambda_c(t)=sS_t.
\]

\begin{lemma}[Shrinking a bad prediction]
\label{lem:ridge-shrinkage-loss}
For every $q\le1$ and $\rho\in[0,1]$, with label one,
\[
(\clip(\rho q)-1)^2
\ge\min\{(1-q)^2,1\}.
\]
\end{lemma}

\begin{proof}
If $0\le q\le1$, then $0\le\rho q\le q\le1$, so no clipping occurs and
$1-\rho q\ge1-q$.  If $q<0$, then $\clip(\rho q)\le0$, whose distance from
the label one is at least one.
\end{proof}

\begin{proof}[Proof of Theorem~\ref{thm:ridge-univariate}]
On active round $t\ge2$, Equation~\eqref{eq:triangular-powered-error} gives
\[
\widehat y_t^0=1-\delta_t,
\qquad
\delta_t=\frac{t(t-1)^\beta}{S_t}\ge0.
\]
Lemma~\ref{lem:triangular-ridge-shrinkage} and
Lemma~\ref{lem:ridge-shrinkage-loss} therefore imply, pointwise in every
$\lambda_t\ge0$,
\[
\ell_t^{\lambda_t}
\ge\min\{\delta_t^2,1\}.
\]

If $m<8$, the empty-history coefficient is zero for every ridge value, so the
first loss is one and
$1\ge m/8\ge c_\alpha m$.  Suppose $m\ge8$ and take
$t=\lceil3m/4\rceil,\ldots,m$.  The endpoint calculation in
Lemma~\ref{lem:triangular-count} gives
\[
\delta_t
\ge L_\alpha,
\qquad
L_\alpha=\frac32\left(\frac58\right)^{2\alpha-2}.
\]
There are at least $m/4$ such rounds, so
\begin{align*}
\Reg_m^{\rm clip}
&\ge\frac m4\min\{1,L_\alpha^2\}\\
&=\frac m4\min\left\{1,
\frac94\left(\frac58\right)^{4\alpha-4}\right\}
\ge c_\alpha m.
\end{align*}
If the second argument of the minimum is at most one, the coefficient in the
last line is exactly
$(9/16)(5/8)^{4\alpha-4}$; otherwise it is $1/4\ge1/8$.

At $\alpha=1$, Lemma~\ref{lem:triangular-count} shows that precisely the
final $\lfloor m/2\rfloor$ active rounds have $\delta_t\ge1$.  Their
Moore--Penrose raw predictions are nonpositive.  Multiplication by
$\rho_t(\lambda_t)\in[0,1]$ keeps them nonpositive, so each has clipped loss
at least one.  Adding the first-round loss proves the sharper bound.

If $d>m$, append identically zero feature coordinates.  The ridge objective
assigns zero to each new transformed coordinate because it does not affect
the residual and a nonzero value would increase the penalty.  Thus active
predictions are unchanged.  If $T>m$, append zero-input, zero-label rounds;
they have zero loss.  This proves the coordinatewise-bounded statement for
arbitrary $(T,d)$.

For the normalized statement, half-scale every active input and label.
Univariate primes remain unchanged by the calculation in
Lemma~\ref{lem:triangular-normalization}.  If $c=1/2$, direct substitution in
Equation~\eqref{eq:appendix-ridge-protocol} gives
\begin{equation}
z_\lambda^{\prime}
=(cA)^\top(c^2AA^\top+\lambda I)^{-1}(cy)
=z_{\lambda/c^2}.
\label{eq:ridge-normalized-parameter}
\end{equation}
By Equation~\eqref{eq:ridge-normalized-parameter}, the scaled raw prediction
is $c$ times the unscaled prediction at ridge value $\lambda/c^2$.  The latter
equals $\rho q_0$ for some $\rho\in[0,1]$, and the Moore--Penrose prediction
$q_0$ is at most one.  The clipping calculation in
Lemma~\ref{lem:triangular-normalization} therefore applies verbatim and
retains at least a factor $c^2=1/4$ of the unscaled ridge loss.  Since the
theorem is uniform over all nonnegative ridge values, this yields
$(c_\alpha/4)\min\{T,d\}$.

Every inequality above is pointwise in $\lambda_t$.  Hence the schedule may
depend on the history, the current input, and an arbitrary fixed realization
of internal randomness without changing the sequence or the bound.
\end{proof}

\subsection{Hadamard ridge certificates for all three rules}

The triangular theorem gives the exact frontier for univariate priming.  The
general resolvent also shows that the common target-mass obstruction survives
ridge for each of the three rules.  Use the completed-pair history from
Equation~\eqref{eq:paired-history}:
\[
X_s=[a,H_s],
\qquad
\bar X_s=\begin{bmatrix}X_s\\-X_s\end{bmatrix},
\qquad
\bar y_s=\begin{bmatrix}a\\-a\end{bmatrix},
\qquad
b=H_s^\top a.
\]

\begin{lemma}[Three ridge certificate costs]
\label{lem:three-rule-ridge-certificate}
Fix $\alpha\ge1$ and $s\ge1$.  For each powered univariate, Pearson, and
multivariate prime on the completed-pair history, there is a nuisance-only
transformed coefficient $u$ satisfying
\[
Bu=\bar y_s,
\qquad
\pi_1^2\lVert u\rVert_2^2\le\frac{s^{2\alpha}}N.
\]
Consequently, for every $\lambda\ge0$,
\begin{equation}
0\le w_{\lambda,1}
\le\frac{s^{2\alpha}/N}{1+s^{2\alpha}/N}.
\label{eq:three-rule-ridge-target-cap}
\end{equation}
The Pearson statement is uniform over every finite totalization.
\end{lemma}

\begin{proof}
Hadamard orthogonality gives the original-coordinate nuisance interpolant
\[
v=\frac1NH_s^\top a=\frac bN,
\qquad
H_sv=a.
\]
By Lemma~\ref{lem:paired-base-primes}, the completed-pair univariate and
Pearson base primes are $(1,b/s)$.  For $b_j\ne0$, set
\[
u_j
=\frac{b_j/N}{\operatorname{sign}(b_j)|b_j/s|^\alpha}
=\frac{s^\alpha}{N}|b_j|^{1-\alpha},
\]
and set $u_j=0$ when $b_j=0$.  Coordinatewise multiplication by the powered
nuisance prime produces $b_j/N$, so the representative transformed system
maps $u$ to $a$ and the full paired system maps it to $\bar y_s$.  Every
nonzero $b_j$ is an integer of magnitude at least one, whence
\[
\lVert u\rVert_2^2
\le N(s^\alpha/N)^2
=s^{2\alpha}/N.
\]
The powered target prime is one.  All active Pearson correlations are defined
on a nonempty paired history, so no totalized value enters.

For multivariate priming, Lemma~\ref{lem:paired-base-primes} gives
\[
p^{\LLS}=\frac{(s,b)}{N+s},
\qquad
\pi_1=\left(\frac{s}{N+s}\right)^\alpha.
\]
For $b_j\ne0$, use
\[
u_j=\frac{(N+s)^\alpha}{N}|b_j|^{1-\alpha},
\]
and set the zero-correlation coordinates to zero.  The powered nuisance
diagonal again maps $u$ to $b/N$, so $Bu=\bar y_s$, and
\[
\lVert u\rVert_2^2\le\frac{(N+s)^{2\alpha}}N.
\]
Multiplying by the squared target prime cancels $(N+s)^{2\alpha}$ and gives
$s^{2\alpha}/N$.  Finally,
Equation~\eqref{eq:three-rule-ridge-target-cap} follows from
Equation~\eqref{eq:ridge-target-cap}.
\end{proof}

\begin{proposition}[Rule-specific paired-Hadamard ridge lower bound]
\label{prop:ridge-hadamard}
Fix $\alpha\ge1$, one of the three powered prime rules, a deterministic ridge
schedule measurable with respect to the completed data history, and $M\ge1$.
Choose a power of two $N$ with
$8M^{2\alpha}\le N<16M^{2\alpha}$ and set $T=2M$, $d=N+1$.  There is a
deterministic realizable paired-Hadamard sequence, fixed before replay, on
which
\begin{equation}
\Reg_T^{\rm clip}\ge\frac{64M}{81}=\frac{32T}{81}.
\label{eq:ridge-hadamard-lower-bound}
\end{equation}
The sequence may depend on the rule and ridge policy.  For Pearson priming,
the same charged losses hold for every finite totalization.
\end{proposition}

\begin{proof}
After $s$ completed pairs, the next nuisance row is the fixed Hadamard row
$h_{s+1}$.  For $s=0$, the historical coefficient is zero and either query
sign loses one.  For $1\le s<M$, Lemma
\ref{lem:three-rule-ridge-certificate} and the choice of $N$ give
\[
0\le w_{\lambda,1}
\le\frac{1/8}{1+1/8}
=\frac19.
\]
Because the ridge schedule is data-history measurable, this coefficient is
fixed before the current input and its target sign are revealed.  Let
$c=h_{s+1}^\top w_{\lambda,2:d}$.  The two candidate raw predictions are
$c+w_{\lambda,1}$ and $c-w_{\lambda,1}$ with labels $+1$ and $-1$.
Lemma~\ref{lem:two-sign} yields
\begin{equation}
\frac{L_s(+1)+L_s(-1)}2
\ge(1-w_{\lambda,1})^2
\ge\frac{64}{81}.
\label{eq:ridge-hadamard-two-sign}
\end{equation}
Choose a maximizing sign, breaking ties toward $+1$, and append that query
followed by its exact negative.  Finite recursion compiles the full sign
string before replay.  Summing the $M$ query losses and discarding all
nonnegative recovery losses proves
Equation~\eqref{eq:ridge-hadamard-lower-bound}.  Every label is the first
coordinate, so $e_1$ has zero loss.

The empty-history Pearson query uses the zero coefficient independently of
the fill.  Only the first uncharged recovery can see a one-sample undefined
correlation.  Every later charged state consists of completed opposite pairs,
where all active variances are positive.  Because the ridge schedule depends
only on the data history, the compiled charged losses are uniform over finite
Pearson totalizations.
\end{proof}

\begin{corollary}[Common random-sign ridge obstruction]
\label{cor:ridge-random-sign}
Under the same $(\alpha,M,N,H)$, take independent uniform Rademacher query
signs, independently of the learner's random tape, and append the exact
opposite after each query.  Every one of the three powered rules and every
possibly randomized past-only ridge schedule satisfies
\[
\mathbb E\Reg_T^{\rm clip}\ge\frac{64M}{81}.
\]
The same bound holds for a pre-input switcher, a pre-input convex mixture of
the three coefficient vectors, or a pre-input convex mixture of their
individually clipped predictions, all at the common power $\alpha$.
\end{corollary}

\begin{proof}
Condition on the completed-pair history and the learner's random tape.  The
current sign remains uniform, whereas every component coefficient, ridge
value, and mixture weight is fixed.  Each component target weight lies in
$[0,1/9]$ by Lemma~\ref{lem:three-rule-ridge-certificate}; a convex
coefficient mixture has the same property, so
Equation~\eqref{eq:ridge-hadamard-two-sign} applies.

For a mixture of individually clipped predictions, write
\[
r_R^+=\clip(c_R+w_R),
\qquad
r_R^-=\clip(c_R-w_R).
\]
Monotonicity and one-Lipschitzness of clipping give
$0\le r_R^+-r_R^-\le2w_R$.  Let $r^+$ and $r^-$ be their convex mixtures,
and put
\[
m_0=\frac{r^++r^-}{2},
\qquad
\delta=\frac{r^+-r^-}{2}.
\]
Then $0\le\delta\le1/9$ and the conditional two-sign loss is
\[
\frac{(r^+-1)^2+(r^-+1)^2}{2}
=m_0^2+(1-\delta)^2
\ge\frac{64}{81}.
\]
Thus both mixture semantics incur conditional expected query loss at least
$64/81$.  The tower property and nonnegativity of recovery losses complete
the proof.
\end{proof}

The deterministic result requires the ridge policy to be fixed before the
current target sign is visible; unlike Theorem~\ref{thm:ridge-univariate}, it
does not cover a current-input-dependent ridge choice.  The regularized
coefficient vectors of the three rules also need not coincide, so the
deterministic witness in Proposition~\ref{prop:ridge-hadamard} is
rule--policy specific.  The common statement in
Corollary~\ref{cor:ridge-random-sign} is distributional.

%% file: appendix/G_activation_experiment.tex
\section{Frozen-activation experiment details}
\label{app:activation-experiment}

\paragraph{Representation and target selection.}
We use the input to the MLP down-projection (the post-gating activation) at the
last nonpadding token of Qwen2.5-7B-Instruct layers 6, 13, and 20.  Text inputs
are Alpaca instructions and optional inputs; model responses are unused.  All
four data pools are disjoint.  Coordinate $i$ is divided by its calibration
95th absolute percentile (floored at $10^{-4}$) and clipped to $[-1,1]$.
Eligibility requires standard deviation at least $0.08$, a 5--95 percentile
range of at least $0.25$, and at least 32 distinct observed values.  Before
computing outcomes, a seeded
permutation selects six eligible targets per layer; another orders each
target's remaining neurons into nested nuisance sets.

\paragraph{Online protocol and metrics.}
For target neuron $j$, each input is the selected activation vector and the
label is $y_t=x_{t,j}$, so the fixed $e_j$ has zero loss.  Strict-past
predictions are clipped to $[-1,1]$, and normalized regret is cumulative loss
divided by $\sum_t y_t^2$.  We test the three unit-power rules from
Section~\ref{sec:setup}, with finite unit fill for undefined Pearson primes.
For each target, one order is compiled offline against univariate priming,
saved, and replayed unchanged for all methods.  The compiler uses only the
fixed activation records and strict-past predictions; it does not alter the
features or labels.

\paragraph{Recovery result.}
The mechanism criterion requires normalized regret at least $0.25$, a
second-half target-weight statistic at most $0.5$, historical residual at
most $10^{-5}$, and prime-weighted nuisance cost
$\pi_j^2\lVert B^\dagger y\rVert_2^2\le1$.  Success was defined as meeting
all four conditions on at least 12 targets spanning at least two layers.  The
criterion passes for 16 targets across all three layers.  Every compiled target exceeds
the regret threshold for all three rules; their median regrets are $1.181$,
$0.355$, and $0.641$.  The comparator error is zero and the largest historical
residual is $2.367\times10^{-13}$.

\paragraph{Fresh dimension--horizon diagnostic.}
The fresh test uses 1,024 nonoverlapping records, two fixed orders, nuisance
widths $8,16,\ldots,2{,}048$, and horizons $16,32,48$.  Order outcomes are
averaged within targets.  The prespecified $T=48$ comparison of widths $16$ and
$2{,}048$ requires regret to rise by $0.10$, target weight to fall by $0.20$,
and nuisance residuals to be respectively above $10^{-3}$ and below $10^{-5}$.
All 18 targets pass.  Median regret rises by $0.619$ (target-bootstrap 95\%
interval $[0.541,0.672]$) and target weight falls by $0.809$
($[0.784,0.857]$).  At width $2{,}048$, row-permuted and Rademacher controls
have median regrets $0.857$ and $0.810$, versus $0.773$ for real activations on
the corresponding subset.

\paragraph{Cross-generation replication.}
After completing the Qwen2.5 experiment, we apply the same protocol and
prespecified outcome thresholds to the official Qwen3.8-27B checkpoint
\citep{qwen38modelcard}.  We use post-gating activations from layers 14, 30,
and 46, matching the relative depths of Qwen2.5 layers 6, 13, and 20.  The split
contains 384 calibration, 384 tuning, and 1,024 evaluation records.  The
protocol seed, calibration-only eligibility rules, six targets per layer,
nuisance dimensions $128,512,2{,}048$, horizon 48, five random-order seeds, and
offline compiled order are unchanged.

The complete mechanism criterion passes for 17 of 18 Qwen3.8 targets: 6 of 6
in layers 14 and 30, and 5 of 6 in layer 46.  The sole failure, neuron 7052 in
layer 46, has normalized regret $1.215$ but second-half target weight $0.507$,
just above the prespecified cutoff $0.5$.  Over the five registered random
orders, median normalized regrets are $0.805$, $0.796$, and $0.774$ for
univariate, multivariate, and Pearson priming.  The one-sparse comparator has
exactly zero error, the largest historical relative residual is
$2.401\times10^{-13}$, and an independent SVD recomputation of nine registered
trajectories agrees in per-round loss to at most $8.01\times10^{-8}$.

These random-order medians are distinct from the compiled-order values in
Table~\ref{tab:qwen-replication}.  For two fixed targets per layer,
target-only median regret on compiled orders
is $0.029$ for each rule.  Adding real nuisance coordinates gives medians
$1.238$, $0.303$, and $0.587$; independently row-permuting each nuisance
coordinate gives $0.799$, $0.306$, and $0.555$.  As a post-hoc sparse-recovery
check, exact best-one-sparse ERM selects the target throughout the second half
of every registered sequence and has median normalized regret $0.058$ on
compiled orders.  This replication diagnoses the same activation-space
mechanism without asserting that Qwen internally uses the priming update.

\begin{table}[!htbp]
  \centering
  \small
  \setlength{\tabcolsep}{5.5pt}
  \caption{\textbf{Cross-generation frozen-activation replication.}
  Entries are median normalized regret on each model's fixed compiled orders.
  The criterion column reports targets satisfying the complete prespecified
  mechanism signature.  Targets and compiled orders are model specific, so
  the table is a replication check rather than a model ranking.}
  \label{tab:qwen-replication}
  \begin{tabular}{lcccc}
    \toprule
    Frozen representation & Criterion & Univariate & Multivariate & Pearson \\
    \midrule
    Qwen2.5-7B & $16/18$ & $1.181$ & $0.355$ & $0.641$ \\
    Qwen3.8-27B & $17/18$ & $1.223$ & $0.303$ & $0.593$ \\
    \bottomrule
  \end{tabular}
\end{table}

\begin{figure}[H]
  \centering
  \includegraphics[width=\linewidth]{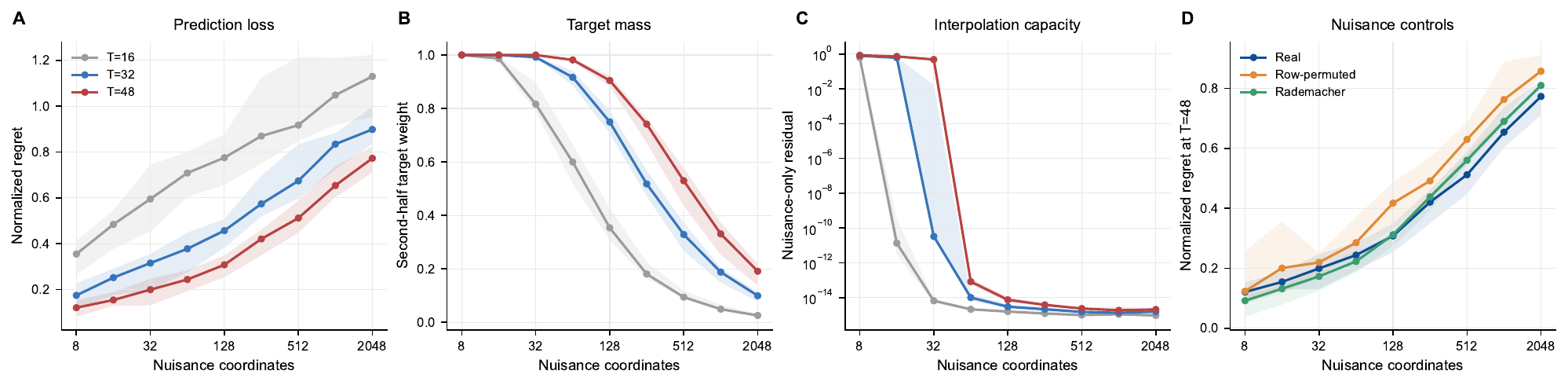}
  \caption{\textbf{The Qwen2.5 mechanism diagnostic.}
  Lines and bands are target-level medians and interquartile ranges: 18 targets
  in A--C and six fixed control targets in D.  As nuisance width grows, regret
  rises (A), target weight falls (B), and nuisance-only interpolation becomes
  exact (C); permuted and Rademacher controls reproduce the regret trend (D).}
  \label{fig:activation-phase}
\end{figure}